%% file: arxiv.tex
\documentclass{article}

\usepackage[preprint]{neurips_2024}

\usepackage[utf8]{inputenc} 
\usepackage[T1]{fontenc}    
\usepackage{hyperref}       
\usepackage{url}            
\usepackage{booktabs}       
\usepackage{amsfonts}       
\usepackage{nicefrac}       
\usepackage{microtype}      
\usepackage{amsmath}
\usepackage{amssymb}
\usepackage{mathtools}
\usepackage{amsthm}
\usepackage{bbm}
\usepackage{enumitem}
\usepackage{wrapfig}
\RequirePackage{algorithm}
\usepackage{lipsum}

\newcommand\blfootnote[1]{%
  \begingroup
  \renewcommand\thefootnote{}\footnote{#1}%
  \addtocounter{footnote}{-1}%
  \endgroup
}

\usepackage{algorithmic}

\usepackage{hyperref}
\usepackage[dvipsnames]{xcolor}
\definecolor{OliveGreen}{rgb}{0,0.56,0}
\hypersetup{
    colorlinks=true,
    linkcolor=blue,
    filecolor=magenta,      
    urlcolor=blue,
    citecolor= OliveGreen
}

\input{macro}

\theoremstyle{plain}
\newtheorem{theorem}{Theorem}[section]
\newtheorem{proposition}[theorem]{Proposition}
\newtheorem{lemma}[theorem]{Lemma}

\theoremstyle{definition}
\newtheorem{definition}[theorem]{Definition}
\newtheorem{assumption}[theorem]{Assumption}
\theoremstyle{remark}

\title{FLIPHAT: Joint Differential Privacy\\ for High Dimensional Sparse Linear Bandits}

\author{%
  Sunrit Chakraborty *\\
  Department of Statistics\\
  University of Michigan\\
  Ann Arbor, MI 48105, USA \\
  \texttt{sunritc@umich.edu} \\
  \And
  Saptarshi Roy *\\
  Department of Statistics\\
  University of Michigan\\
  Ann Arbor, MI 48105, USA \\
  \texttt{roysapta@umich.edu} \\
  \AND
  Debabrota Basu \\
  Inria Lille – Nord Europe, équipe Scool \\
  Villeneuve d’Ascq, France \\
  \texttt{debabrota.basu@u.nus.edu} \\
}

\begin{document}

\maketitle

\begin{abstract}
   High dimensional sparse linear bandits serve as an efficient model for sequential decision-making problems (e.g. personalized medicine), where high dimensional features (e.g. genomic data) on the users are available, but only a small subset of them are relevant. Motivated by data privacy concerns in these applications, we study the joint differentially private high dimensional sparse linear bandits, where both rewards and contexts are considered as private data. First, to quantify the cost of privacy, we derive a lower bound on the regret achievable in this setting. To further address the problem, we design a computationally efficient bandit algorithm, \textbf{F}orgetfu\textbf{L} \textbf{I}terative \textbf{P}rivate \textbf{HA}rd \textbf{T}hresholding (FLIPHAT). Along with doubling of episodes and episodic forgetting, FLIPHAT deploys a variant of Noisy Iterative Hard Thresholding (N-IHT) algorithm as a sparse linear regression oracle to ensure both privacy and regret-optimality. We show that FLIPHAT achieves optimal regret in terms of privacy parameters $\epsilon, \delta$, context dimension $d$, and time horizon $T$ up to a linear factor in model sparsity and logarithmic factor in $d$. We analyze the regret by providing a novel refined analysis of the estimation error of N-IHT, which is of parallel interest.
\end{abstract}\vspace*{-1em}

\blfootnote{* Equal contribution.}

\input{arxiv_body}

\begin{ack}
    D. Basu acknowledges the ANR JCJC project REPUBLIC (ANR-22-CE23-0003-01), the PEPR project FOUNDRY (ANR23-PEIA-0003), and the Inria-Japan associate team RELIANT for supporting the project. 
\end{ack}

\bibliography{ref}
\bibliographystyle{apalike}


\appendix

\input{arxiv_appendix}


\newpage

\end{document}

%% file: macro.tex
\def\abs#1{\left|#1\right|}

\def\argmin{{\arg\min}}
\def\argmax{{\arg\max}}

\def\bB{\mathbf{B}}

\def\bH{\mathbf{H}}

\def\bW{\mathbf{W}}
\def\bX{\mathbf{X}}

\def\ba{\mathbf{a}}

\def\br{\mathbf{r}}

\def\bw{\mathbf{w}}

\def\by{\mathbf{y}}

\def\bbE{\mathbb{E}}

\def\bbN{\mathbb{N}}

\def\bbP{\mathbb{P}}

\def\bbR{\mathbb{R}}

\def\cA{\mathcal{A}}
\def\cB{\mathcal{B}}
\def\cC{\mathcal{C}}
\def\cD{\mathcal{D}}
\def\cE{\mathcal{E}}

\def\cG{\mathcal{G}}
\def\cH{\mathcal{H}}

\def\cL{\mathcal{L}}
\def\cM{\mathcal{M}}
\def\cN{\mathcal{N}}
\def\cO{\mathcal{O}}
\def\cP{\mathcal{P}}
\def\cQ{\mathcal{Q}}
\def\cR{\mathcal{R}}

\def\cZ{\mathcal{Z}}

\def\xmax{x_{\sf max}}
\def\bmax{b_{\sf max}}

\def\eps{{\epsilon}}

\def\ind{{\mathbbm{1}}}

\def\KL{{\sf KL}}

\def\norm#1{\left\|#1\right\|}

\def\Pr{{\bbP}}

\def\supp{{\sf supp}}
\def\clip{{\sf clip}}

\def\pr{{\bbP}}

\def\TV{{\sf TV}}
\def\Lap{{\sf Lap}}

\def\sfA{{\sf A}}

\newcommand{\innerprod}[1]{\left\langle#1\right\rangle}
\newcommand{\ceil}[1]{\left\lceil#1\right\rceil}
\newcommand{\floor}[1]{\left\lfloor#1\right\rfloor}

%% file: arxiv_body.tex
\section{INTRODUCTION}
Multi-armed bandits, in brief \textit{bandits}, constitute an archetypal model of sequential decision-making under partial information~\citep{lattimore2018bandit}. In bandits, a learner sequentially chooses one decision from a collection of actions (or arms). In return, it observes only the reward corresponding to the chosen action with the goal of maximizing the accumulated rewards over a time horizon.
Additionally, the available decisions at any point of interaction might depend on some side information about the actions, encoded as a feature vector. In linear contextual bandits (LCB) \citep{auer2002using,abe2003reinforcement}, the learner observes such a feature vector (aka \textit{context}) for each action before taking a decision, and the rewards are generated by a linear function of the chosen context and an unknown parameter vector. 
LCBs are widely studied and deployed in real-life applications, such as online advertising~\citep{schwartz2017customer}, recommendation systems~\citep{silva2022multi}, investment portfolio design~\citep{silva2022multi}, and personalized medicine~\citep{peng2021machine}. 

In many of these applications, the personal information used as contexts often has a large dimension, which can be modeled by viewing the contexts as high-dimensional vectors. This is relevant in the modern big data era, where the applications usually deal with many features. However, the established regret lower bound in LCB polynomially depends on the context dimension~\citep{dani2008stochastic, chu2011contextual}. Thus, additional structural conditions, such as sparsity, are needed to be enforced to make the problem tractable. Sparsity aligns with the idea that among these features, possibly only a small subset of them are relevant to generate the reward. This setting is called the \textit{high-dimensional sparse linear contextual bandit} (SLCB). Recently, SLCB has received significant attention~\citep{wang2018minimax,bastani2020online,HaoLattimore2020,chakraborty2023thompson} and have been useful in practical applications~\citep{miao2022online,ren2024dynamic}.

\textbf{Data Privacy in SLCB.} We observe that the applications of SLCB, e.g. recommender systems and personalized medicines, involve sensitive data of individuals. This naturally invokes data privacy concerns and requires development and analysis of data privacy-preserving SLCB algorithms. 

\textit{Example 1.} One of the earliest successful application of linear contextual bandits is personalized news article recommendation and its validation on Yahoo! Today Module by~\citet{li2010contextual}. In the dataset collected from Yahoo! platform, each user is represented by a 1193-dimensional features but for a successful deployment they are reduced to 5 dimensions. 
The feature of every user include their gender, age, geographical address, and even behavorial quantifiers. 
These are sensitive data as per GDPR~\citep{voigt2017eu} and California Consumer Privacy Act~\citep{ccpa}. Thus, protecting data privacy in this problem, and designing efficient and private SLCB algorithms are fundamentally necessary.

\textit{Example 2.} A recent example of deploying SLCB with user-data is optimizing click-through rates in the Tencent search advertising
dataset~\citep{tencent}. Though each context is 509256 dimensional, it lends itself to a sparse structure~\citep{wangtencent}. 
Here, also the context incorporates sensitive information, like gender, age etc., which are protected by privacy laws. Thus, private SLCBs are a natural requirement here.

\textit{Example 3.} 
Another application of SLCB is a mobile health application that recommends a personalized treatment plan (actions) to each patient (user) based on her personal information, such as age, weight, height, medical history etc.~\citep{tewari2017ads,killian2023robust}. 
Here, for every user and personalized treatment plan contexts are generally obtained through a feature representation, where the dimension is often large but there are only fewer relevant features. 
The algorithm uses patients' data including treatment outcomes and medical contexts of patients, which are sensitive information as per GDPR, CCPA, and HIPAA~\citep{annas2003hipaa}. Hence, this is a natural application that motivates the study of efficient and privacy-preserving SLCB algorithms.

\textbf{Related Work: Differential Privacy in Bandits.} Problems like the ones in the aforementioned examples motivate us to study data-privacy preserving SLCB algorithms. Specifically, we follow Differential Privacy (DP)~\citep{dwork2006differential} as the data privacy framework. DP aims to ensure that while looking into the output of an algorithm, an adversary is unable to conclude whether a specific user was in the input dataset. DP has emerged as the gold standard of data privacy. It has been extensively studied in various domains with the goal of safe-guarding sensitive data from adversarial attacks~\citep{dwork2014algorithmic,Abadi_MomentAccountant,McMahan_Privacy}.
Impact of DP and designing efficient DP-preserving algorithms have also been studied for different settings of bandits, such as stochastic bandits~\citep{Mishra2015NearlyOD, tossou2016algorithms, sajed2019optimal, hu2022near,azize2022privacy_dpband}, adversarial bandits~\citep{tossou2017achieving} and LCB~\citep{shariff2018differentially,neel2018mitigating, hanna2024differentially,azize2023privacy}. 

The literature also reveals that different notions of DP are possible in bandits depending on what we consider as the sensitive input to be saved and what are the public output. In comparison to the supervised learning literature, additional difficulties arise due to the complex interaction between the policy employed by the learner, and the stochastic nature of the contexts and rewards. \textit{Local}~\citep{duchi2013localextended} and \textit{global} \citep{basu2019differential} definitions of DP are extended to bandits, where we aim to keep the rewards generated during interactions private.
In local DP, users already send to the learner privatized rewards as feedback. In global DP, the learner gets true rewards but it publishes actions in a way that the rewards stay private.
For LCBs, compared to reward privacy, \cite{shariff2018differentially} study a stronger form of privacy that aims to keep both the rewards and the contexts private while publishing the actions. This is called Joint Differential Privacy (JDP), which is further strengthened by~\cite{vietri2020private}. 
In this paper, \textit{we follow $(\epsilon, \delta)$-JDP of \citep{vietri2020private} as our privacy framework}, since it best aligns with our interest of protecting both users' information and the outcomes collected over all the interactions.

We observe that though there are multiple algorithms for  $(\epsilon, \delta)$-JDP (Definition~\ref{def: JDP}) satisfying LCBs, they are tailored for the low dimensional setting and scales polynomially with the context dimension $d$. In contrast, the study of JDP in SLCB setting is still unexplored. This leads us to two questions:\\
1. \textit{What is the cost of ensuring $(\epsilon, \delta)$-JDP in high-dimensional sparse linear contextual bandits?}\\
2. \textit{Can we design a privacy-preserving SLCB algorithm that achieves the minimal cost due to $(\epsilon, \delta)$-JDP?}

\textbf{Contributions.} Our study of these questions leads us to the following contributions.

    1. \textit{Regret Lower Bound}:
    We derive a {  \textit{problem independent}} regret lower bound for SLCBs with $(\epsilon,\delta)$-JDP. Our lower bound demonstrates a phase transition between the hardness of the problem depending on the privacy level, similar to~\citep{azize2023privacy} (reward-private low-dimension LCB).
    Compared to linear dependence on $d$ \citep{azize2023privacy}, our lower bound depends linearly on { {$\sqrt{s^*} \log d$}} for a fixed privacy budget, where $s^*$ is the problem sparsity parameter and $d$ is the context dimension.
    
    2. \textit{Algorithm and Regret Analysis}: We propose a novel SLCB algorithm, namely FLIPHAT, that uses N-IHT as a key component, and prove its $(\epsilon,\delta)$-JDP guarantee. We analyze the expected regret for FLIPHAT and demonstrate that in the problem-independent case, the upper bound matches our lower bound up to  { {$s^*$} and} logarithmic factors in $T$. We also numerically demonstrate its logarithmic dependence on $d$ across privacy regimes.
   
    3. \textit{Interaction of Privacy and Hardness of Bandit:} In Theorem \ref{thm: regret}, we identify a unique interplay between two competing hardness: privacy and the internal hardness of the bandit instance. We show that as the internal hardness of bandit decreases, the hardness due to privacy becomes more prevalent and persists for wider (and larger) values of the privacy parameters, and vice-versa. To the best of our knowledge, this is a previously unidentified phenomenon, and demands careful understanding.

     4. \textit{Improved analysis of N-IHT}: We provide a refined analysis for Peeling-based~\citep{dwork2021differentially} empirical risk minimization under non i.i.d. setting. This yields a bound on estimation error of N-IHT under the SLCB setting, and it is used in the regret analysis of FLIPHAT. The results and tools developed here might be of independent interest and can be applied in settings where N-IHT is deployed.



\section{PROBLEM FORMULATION}\label{sec: problem}

A high-dimensional \textit{sparse linear stochastic contextual bandit (SLCB)} consists of $K\in \mathbb{N}$ arms. At time $t\in [T]$\footnote{For $K \in \bbN$, $[K]$ denotes the set $\{1, \ldots, K\}$.}, context vectors $\{x_i(t)\}_{i\in [K]}$ are revealed for every arm $i$. 
For every $i\in[K]$, we assume $x_i(t)\in \bbR^d$, and are sampled i.i.d.\ from some unknown distribution $P_i$. 
At every step $t$, an action $a_t \in [K]$ is chosen by the learner and a noisy reward $r(t)$ is obtained through a linear model:
\begin{align}\label{eq: base model}
    r(t) = x_{a_t}(t)^\top \beta^* + \epsilon(t)\,.
\end{align}
Here, $\beta^*\in\bbR^d$ is the unknown true signal, which is $s^*$-sparse i.e., $\norm{\beta^*}_0 = s^*$ with $s^* \ll d$. $\{\epsilon(t)\}_{t\in[T]}$ are independent sub-Gaussian noise, which are independent of all the other processes. Furthermore, we assume boundedness assumption on the contexts and parameter, i.e., $\norm{x_i(t)}_\infty \le \xmax \;\forall (i,t) \in [K]\times [T]$ almost surely, and $\norm{\beta^*}_1 \le \bmax$. Such assumptions are common in prior LCB literature, \citep{HaoLattimore2020, li2021regret, azize2023privacy} and are crucial for privacy guarantees.

Now, to formalize the notion of performance of a learner, we define the best action in hindsight at the step $t$ as $a_t^*:= \argmax_{i \in [K]} x_i(t)^\top \beta^*$. Then, the goal is to design a sequential algorithm (policy) $\pi$. At each step $t$, $\pi$ takes the observed reward-context pairs until time $t-1$ as input and outputs a distribution $\pi_t$ over the action space $[K]$. The learner selects an action $a_t\sim \pi_t$ using the policy $\pi$. Then, we define regret of policy $\pi$ as\footnote{We sometimes drop the subscript $\pi$ in regret, when the dependence of on the policy is clear from context.} 
\begin{align}\label{eq:regret}
    R_{\pi}(T) := \sum\nolimits_{t \in [T]} \left[x_{a_t^*}(t)^\top \beta^* - x_{a_t}(t)^\top \beta^*\right]\, .
\end{align}
The goal of the learner is to construct a SLCB policy $\pi$ to maximize the cumulative reward till time $T$, or equivalently, minimize $R_{\pi}(T)$.

\textbf{Joint DP for SLCB.} Now, we concretely define the notion of DP in SLCB. 
It is well known that the original notion of DP~\citep{dwork2006calibrating} is not ideal in contextual bandit setting due to the inevitability of linear regret~\citep[Claim 13]{shariff2018differentially}.
Therefore, we consider the notion of JDP~\citep{vietri2020private}. In JDP, both rewards and contexts are considered as sensitive information about the users which are used by $\pi$ as input and should be kept private from any external adversary. 

To formalize further, let us denote by $\br_t:= (r_1(t), \ldots, r_K(t))$ the collection of rewards, and $\cC_t:= (x_1(t), \ldots, x_K(t))$ the collection of user contexts at time $t$. At the beginning of the game, the adversary prefixes the whole collection of rewards $\{\br_t\}_{t \in [T]}$ and contexts $\{\cC_t\}_{t\in [T]}$.  
In the $t$-th step, the learner observes the contexts $\cC_t$ and chooses an action $a_t$ using only the past \textit{observed} rewards and contexts. Then, adversary reveals the reward $r_{a_t}(t)$ and the game continues.
Under this paradigm, one can define the whole input dataset as the collection of all the rewards and contexts. Specifically, we consider $\cD:= \{(\br_t, \cC_t)\}_{t \in [T]}$ as the input data to a bandit algorithm $\cA$.  
A randomized bandit algorithm $\cA$ maps every reward-context sequence $\cD$ of length $T$ to a (random) sequence of actions $a:=(a_1,\dots,a_T)\in [K]^T$. Therefore, $\cA(\cD)$ can be imagined as a random vector taking values in $[K]^T$. 
Let $\cA_{-t}(\cD)$ denote all the outputs excluding the output in the $t$-th step of the interaction of $\cA$ with the bandit instance. 
$\cA_{-t}(\cD)$ captures all the outputs that might leak information about the user at the $t$th time point in interactions after the $t$th round, as well as all the outputs from the previous episodes where other users could be submitting information to the agent adversarially to condition its interaction with the $t$th user. Finally, we say that two input datasets $\cD$ and $\cD^\prime := \{(\br^\prime_t, \cC^\prime_t)\}_{t \in [T]}$ are $\tau$-neighbors, if for all $t \ne \tau$ we have $(\br_t, \cC_t) = (\br^\prime_t, \cC^\prime_t)$, and $(\br_\tau, \cC_\tau) \ne (\br^\prime_\tau, \cC^\prime_\tau)$.
\begin{definition}[JDP in LCB]
\label{def: JDP}
    A randomized bandit algorithm $\cA$ is $(\epsilon, \delta)$-JDP if for any $\tau \in [T]$ and any pair of $\tau$-neighboring sequences $\cD$ and $\cD^\prime$, and any subset $E\subset [K]^{T-1}$, it holds that  $$\pr(\cA_{- \tau}(\cD) \in E) \le e^\epsilon \pr(\cA_{- \tau}(\cD^\prime) \in E) + \delta.$$
\end{definition}
Although the above JDP definition is valid for any randomized contextual bandit algorithm $\cA$,
it is useful to point out that in our setup, the bandit protocol $\cA$ is solely defined by the inherent policy $\pi$. 

\begin{table*}[t]
  \caption{Comparison with prior works on private LCB (under no-margin condition)}
  \label{table: prior work}
  \centering
  \small{
  \begin{tabular}{llll}
    \toprule
        Paper & Settings & Contexts  & Regret Bound\\
    \midrule

    \cite{shariff2018differentially} & LCB + JDP  & adversarial & $\tilde{O}(d \sqrt{T/\epsilon})$\\
    \cite{zheng2020locally} & LCB + LDP & adversarial & $\tilde{O}((dT)^{3/4}/\epsilon)$\\
    \cite{han2021generalized} & LCB + LDP  & stochastic & $\tilde{O}(d^{3/2}\sqrt{T/\epsilon})$\\
    \cite{garcelon2022privacy} & LCB + JDP & adversarial & $\tilde{O}(dT^{2/3}/\epsilon^{1/3})$\\
    \cite{hanna2024differentially} & LCB + reward-DP & adversarial & $\tilde{O}(d \sqrt{T} + d^2/\epsilon)$ \\
    \cite{chowdhury2022shuffle} & LCB + LDP & adversarial & $\tilde{O}(d T^{3/5} + d^{3/4}T^{3/5}/\epsilon^{1/2})$\\\cite{azize2023privacy} & LCB + reward-DP & stochastic & $\tilde{O}(\sqrt{dT} + d/\epsilon)$ \\
    \textbf{This paper} & \textbf{Sparse LCB + JDP}   & stochastic & $\tilde{O}(s^*\sqrt{T \log d} \vee (s^{* 3/2} \log d)/\epsilon)$\\
    \bottomrule
  \end{tabular}}
\end{table*}

\textbf{Connection with Prior Works on Private LCB.} 
\cite{shariff2018differentially} first studied LCB under $(\epsilon,\delta)$-JDP. Their proposed UCB-based algorithm enjoys $\tilde{O}(d\sqrt{T/\eps
})$\footnote{$\tilde{O}(\cdot)$ hides the poly-logarithmic dependence in $T$ and $d$.} regret bound. However, \textit{their notion of JDP is slightly weaker as it only considers the privacy of future actions}. \cite{garcelon2022privacy} considered a stronger definition of JDP (same as Definition \ref{def: JDP}), and proposed an algorithm with $\tilde{O}(d T^{2/3}/\epsilon^{1/3})$ regret. Still, achieving the known lower bound under JDP~\citep{he2022reduction}, i.e. $\Omega(\sqrt{dT\log K} + \frac{d}{\epsilon+\delta})$, is an open question. 
There is a parallel line of works~\citep{zheng2020locally, han2021generalized, chowdhury2022distributed} that considers Local Differential Privacy (LDP) under the LCB setting. LDP is a much stronger notion of privacy compared to the central nature of JDP. Therefore, the algorithms under LDP typically suffer a worse regret guarantee, i.e. $\Omega(\sqrt{dT}/\eps)$. 
In another stream, \cite{hanna2024differentially} and \cite{azize2023privacy} considered the reward-DP model under the LCB setting, which only aims to privatize the reward stream. Under an adversarial (contexts) setting, \cite{hanna2024differentially}  obtains $\tilde{O}(d \sqrt{T} + d^{2}/\epsilon)$, and \cite{azize2023privacy} obtains $\tilde{O}(\sqrt{dT} + d/\epsilon)$ regret under stochastic setting that matches the lower bound. However, all these algorithms suffer a polynomial dependence in $d$, i.e. $\Omega(\sqrt{d})$ in non-private term and $\Omega(d)$ in private term, which are not desirable in high-dimension. In our work, we propose a novel algorithm under the SLCB setting that only scales poly-logarithmically on the dimension $d$ under JDP (ref. Table \ref{table: prior work}). Moreover, our regret upper and lower bounds demonstrate a phase transition in the regret in the SLCB setting under JDP, similar to \citep{garcelon2022privacy,azize2022privacy}.

\section{LOWER BOUNDS ON REGRET}\label{sec: lower bound}
The first question that we want to address is what is the inherent cost of privacy in SLCB. Hence, we derive a minimax regret lower bound on the regret achieved by any algorithm ensuring $(\epsilon,\delta)$-JDP for SLCB. The other goal is to understand the cost of privacy on the minimax lower bound compared to the existing lower bounds for SLCB without differential privacy. First, we formally define the minimax regret in this setting. 
\begin{definition}[Minimax regret] We define the minimax regret in the $(\epsilon,\delta)$-JDP SLCB  as $R^{\mathrm{minimax}}_{(\epsilon,\delta)}(T) := \inf_{\pi\in \Pi(\epsilon,\delta)} \sup_{P_{\bX,\br}\in \cB} \bbE[R_{\pi}(T)]$, where the supremum is taken over $\cB$, i.e. the space of SLCB instances, and the infimum is taken over $\Pi(\epsilon,\delta)$, i.e. the space of $(\epsilon,\delta)$-JDP policies. 
\end{definition}
Further details about the minimax regret, the space of bandit instances and the associated probability space are in Appendix \ref{app: lower bound proof new}.
Our first result gives a lower bound on the minimax regret, in terms of the impact of the bandit model parameters $s^{*}, d, T$, and privacy parameters $\epsilon, \delta$.

\begin{theorem}[Lower Bound]\label{thm:lb_edp}
    If $\epsilon, \delta> 0$ and $\epsilon^{2} < \log(1/\delta)$, then for sufficiently large $s^* \log(d/s^*)$ the minimax regret for SLCBs under $(\epsilon,\delta)$-JDP
    \begin{align*}
        & R^{\mathrm{minimax}}_{(\epsilon,\delta)}(T)=\\
        &\Omega\Big(\max \Big\{\underbrace{  \sqrt{\frac{s^* \log^2 (d/s^*) \log(1/\delta)}{\epsilon^2}}}_{{\text{private} \text{and high-d}}}, \underbrace{\sqrt{s^* T \log (d/s^*)}}_{\substack{\text{non-private}\\ \text{and high-d}}}\Big\}\Big).
    \end{align*}
    Additionally, for $\delta=0$, we get $R^{\mathrm{minimax}}_{\epsilon}(T) = \Omega\Big(\max \{s^{*} \log^{3/2} (d/s^*) \epsilon^{-1}, 
    {\sqrt{s^* T\log (d/s^*)} }\rbrace\Big).$
\end{theorem}
\textit{Implications.} (i) This is the first lower bound for SLCBs with JDP. All the existing bounds with JDP are for low-dimensional LCBs with JDP. (ii) Theorem~\ref{thm:lb_edp} indicates that the impact of privacy appears in the regret lower bound when { $\epsilon^2/\log (1/\delta) < \log (d/s^*)/T$}. For $\delta = 0$, the above theorem also states that the cost of privacy becomes prevalent in the regime { $\epsilon < \sqrt{s^*} \log(d/s^*)/\sqrt{T}$}, which matches in spirit with the lower bound of~\citep{azize2022privacy_dpband} for a fixed low-dimensional context set, where the impact of privacy appears for $\epsilon < T^{-1/2}$. Additionally, in SLCB, the dimension also matters for the transition. This phenomenon is also reflected in our regret upper bounds. 

\textit{Proof Technique.} Our proof differs from the existing lower bound techniques for private bandits in \citep{shariff2018differentially,basu2019differential_dpband,azize2022privacy,azize2023privacy}, and non-private SLCB setting \citep{li2021regret}.
Rather our proof depends on converting the regret lower bound problem to an estimation lower bound problem following \cite{he2022reduction}, and combing the information-theoretic techniques for sparse linear models~\citep{duchi2013distance} and { privacy-restricted lower bound techniques in~\citep{azize2023privacy, kamath2022improved}.} 



\section{REFINED ANALYSIS: PEELING FOR $(\epsilon,\delta)$-DP ERM}\label{section: general result}
In this section, we describe and analyze an $(\epsilon,\delta)$-DP algorithm, which can be used for general private empirical risk minimization. While the algorithm appears in \citep{cai2021cost} in sparse regression setting, we present a refined and simpler proof of the utility guarantee such that the results are applicable not only in the standard regression setting with independence across observations, but also to data arising from highly dependent and dynamic bandit environments. 

The well-known \emph{Peeling} algorithm (Appendix~\ref{section: background_privacy}) is used for $(\epsilon, \delta)$-DP top-$s$ selection problem~\citep{dwork2021differentially}.
We denote the output of Peeling applied to a vector $v \in \bbR^d$ as $\cP_s(v ;\epsilon,\delta, \lambda)$, where $\epsilon, \delta$ are privacy parameters and $\lambda$ is the scale parameter of the Laplace noise. For $v\in \bbR^d$, and an index set $S\subseteq[d]$, $v_S\in\bbR^d$ denotes the vector such that $v_S$ matches $v$ on indices in $S$ and is 0, elsewhere. 
Peeling returns a $s$-sparse vector in $\bbR^d$.  It is important to point out that the peeling algorithm is an important building block of N-IHT, which in turn is a crucial component of FLIPHAT. 

\textbf{A Generic Analysis of N-IHT.} Let us consider $n$ observations $(x_1,y_1),\dots,(x_n,y_n)$, 
where $x_i\in \bbR^d$ and $y_i\in\bbR$ for all $i\in[n]$, arising from a joint distribution $\cQ$. We denote the design matrix by $X$ and $y\in\bbR^n$ as the vector of responses. We note that we neither assume any particular model (e.g. linear model), nor independence across observations. Now, we wish to use Algorithm \ref{alg:noisy-IHT} to minimize the empirical risk
$\cL(\theta|X,y) := \frac{1}{n}\sum_{i=1}^n \ell((y_i, x_i);\theta)$ w.r.t. a loss function $\ell(\cdot; \cdot)$.
In Algorithm~\ref{alg:noisy-IHT}, $\clip_R(z) := z\min \{1, R/|z|\}$ and $\clip_R$ on a vector $y$ means a coordinate-wise application (Line 3). $\cP_s$ is the peeling subroutine serving as a private hard thresholding step (Line 4). Finally, $\Pi_C$ is the (Euclidean) projection on the ball $\{v\in\bbR^d\mid \norm{v}_1\leq C\}$, i.e., $\Pi_C(x) = \argmin_{v: \norm{v}_1 \le C} \norm{x -v}_2$ (Line 5).
We note that the clip function and the projection step at the end are crucial for the sensitivity analysis to get a DP guarantee for the entire algorithm. For subsequent reference, we denote N-IHT$(X,y|s,\epsilon,\delta, M, R, B, \eta, C)$ to be the output of Algorithm \ref{alg:noisy-IHT} on data $(X,y)$ with the tuning parameters as mentioned in Algorithm \ref{alg:noisy-IHT}.

Assume that $\norm{x_i}_{\infty}\leq \xmax$ and $\theta^*\in\bbR^d$ is any point such that  $\norm{\theta^*}_1\leq \bmax$ and $\norm{\theta^*}_0=s^*$. Now, we consider three \textit{good} events. First, we define the event $\cE_1$ as the one satisfying the infimum of minimum sparse eigenvalue, i.e. $\inf_{\theta} \phi_{\min}(Cs^*, \nabla^2 \cL(\theta)) \geq \underline{\kappa}$, and the supremum of maximum sparse eigenvalue, i.e. $\sup_{\theta} \phi_{\max}(Cs^*, \nabla^2 \cL(\theta)) \leq \bar{\kappa}$ (ref. Definition~\ref{def: sparse eigen value}). Here, $\nabla^2 \cL$ is the Hessian computed for the data $(X,\clip_R(y))$. Second, $\cE_2 := \left\{\norm{\nabla \cL(\theta^*)}_{\infty} \leq \bar{g}\right\}$, where $\nabla \cL$ is the corresponding gradient. Finally, $\cE_3 := \left\{\max_{m\in[M]} \bW_m \leq \bar{W}_{(\epsilon,\delta)}\right\}.$
Here, $\bW_m :=\sum_{i\in[s]} \Vert\bw_i^{(m)}\Vert_{\infty}^2 + \Vert\tilde{\bw}_{S^m}^{(m)}\Vert_2^2$ is the total amount of noise injected in the $m$-th iteration, where $\bw_i$ and $\tilde{\bw}$ are the two sources of noise introduced by Peeling $\cP_s$ in Algorithm~\ref{alg: peeling}.
With these events, we derive an upper bound on the $\ell_2$ distance of $\hat{\theta}$, i.e. the output of Algorithm~\ref{alg:noisy-IHT}, from $\theta^*$.

\setlength{\textfloatsep}{6pt}
\begin{algorithm}[t!]
\caption{N-IHT}\label{alg:noisy-IHT}
\begin{algorithmic}[1]
\REQUIRE Data $\bX, \by$, sparsity $s$, privacy level $\epsilon, \delta$, number of iterations $M$, truncation level $R$, noise scale $B$, step-size $\eta$, projection level $C$

\STATE Initialize $\theta^0 \in \bbR^d$.

\FOR{$m = 0, 1, \dots, M-1$}
    \STATE $\theta^{m+0.5} = \theta^m - \eta\nabla \cL(\theta^{m}|X, \clip_R(y))$ 
    \STATE $\tilde{\theta}^{m+1} = \cP_s(\theta^{m+0.5};\epsilon/M, \delta/M,\eta B / n)$. 
    \STATE $\theta^{m+1} = \Pi_C(\tilde{\theta}^{m+1})$.
\ENDFOR

\RETURN \textbf{Return} $\hat{\theta}=\theta^{M}$
\end{algorithmic}
\end{algorithm}

\begin{theorem}\label{thm: general result}
     Let $\eta=1/2\bar{\kappa}, M\ge 28\kappa\log(\bmax^2 n)$, $C\ge \norm{\theta^*}_1$ and $s \gtrsim \kappa^2 s^*$. Then, with probability at least $1-\Pr(\cE_1^c)-\Pr(\cE_2^c)-\Pr(\cE_3^c)$, 
    $$\norm{\theta^M - \theta^*}_2^2\leq \frac{1}{n} + C_1\kappa^2\bar{W}_{(\epsilon,\delta)} + C_2 \frac{\kappa^2}{\underline{\kappa}^2} (s+s^*) \bar{g}^2,$$
    for suitably large constants $C_1, C_2>0$, where the probability measure $\Pr$ is that arising from the joint distribution $\cQ$ and the randomness due to privacy. $\kappa=\bar{\kappa}/\underline{\kappa}$ denotes the condition number of the Hessian of loss $\cL$.
\end{theorem}\vspace*{-.5em}
The proof is detailed in Appendix~\ref{app: proof of general result}. For the squared error loss $\ell((y,x);\theta)=(y-x^\top\theta)^2$ and independent samples, the output of Algorithm~\ref{alg:noisy-IHT} with $B=R+\xmax\bmax$ is $(\epsilon,\delta)$-DP~\citep{cai2021cost}. 
In other cases, one can tune $B$ depending on the sensitivity requirement to obtain DP guarantees.

\textit{Generic Implications.} The theorem is general in the sense that it holds for $(X,y)$ arising from any joint distribution $\cQ$ and $\ell(\cdot; \cdot)$ being any loss function. The Laplace noise (in Peeling) is not mandatory to use this result. The first term in the upper bound goes to $0$ as the sample size increases. The third term, controlled by the behavior of the gradient of the loss at the target point, typically is the statistical bound. The second term is the cost we pay for ensuring $(\epsilon,\delta)$-DP for the output. Clearly, we should not expect $\norm{\theta^M-\theta^*}_2$ to go to 0, for arbitrary $\theta^*$. However, once we make certain assumptions about the model (e.g., linear model where $\theta^*$ is the true data generating parameter), and the loss structure (e.g., squared error), the result would allow us to `plug-in' standard guarantees about control of the good events in this result and obtain specific estimation bounds.

\vspace*{-.3em}\section{FLIPHAT: ALGORITHM DESIGN \& REGRET ANALYSIS}\label{sec: upper bound}\vspace*{-.3em}
In this section, we propose a novel algorithm FLIPHAT for the SLCB setting, which is guaranteed to be $(\epsilon,\delta)$-JDP. 
Then, we derive an upper bound on its expected regret that matches the regret lower bound from Section~\ref{sec: lower bound} up to logarithmic terms.

\vspace*{-.2em}\subsection{Algorithm Design}\vspace*{-.2em}

We present the pseudocode of FLIPHAT in Algorithm \ref{alg: fliphat}. FLIPHAT has three components. 

\textit{First}, we employ the doubling trick~\citep{besson2018doubling}, i.e. dividing the entire time horizon into episodes of geometrically progressive lengths. Specifically, the $\ell-$th episode begins at step $t_{\ell}=2^{\ell}$. At the beginning of each episode, the underlying algorithm restarts, computes the required estimates, and the episode continues for $N_\ell = 2^{\ell+1} - 2^{\ell} = 2^{\ell}$ steps. In non-private bandit, we use doubling trick to make the algorithm `anytime', but for private bandits, it also helps to reduce the number of times for which we have to calculate any private estimates from the data. Since DP is typically achieved by adding scaled Laplace noise with the newly computed estimates~\citep{dwork2014algorithmic}, doubling ensures that we do not add noise for more than a logarithmic number of times. This is important for regret optimality.

\textit{Second}, to ensure DP, we use the \textit{forgetfulness} technique, i.e. at the beginning of an episode, we only use the data collected from the previous episode to generate the estimate. In the current episode, we do not update the estimate further. This \textit{forgetful} strategy (the observed data from an episode is used just once at the start of the next episode and then forgotten) allows us to guarantee that the overall algorithm is private, if the computed estimates are private. 
The main idea is that a change in a reward and context will only affect the estimate calculated in one episode. This allows parallel composition of DP and stops the aggravated accumulation of Laplace noise over time, yielding better regret guarantees.

\textit{Third}, we use N-IHT (Algorithm \ref{alg:noisy-IHT}) as the private sparse linear regression oracle to estimate the linear parameter from the observed samples. 
At the beginning of $(\ell+1)$-th episode (line 8), we feed N-IHT the data $(X_\ell,y_\ell)$ consisting of the chosen contexts and observed rewards from the $\ell$th episode (line 13). Other tuning parameters $(s,\epsilon,\delta,\eta,C)$ are not explicitly mentioned as they stay same for all $\ell$'s. 

\begin{algorithm}[t!]
\caption{FLIPHAT}\label{alg: fliphat}
\begin{algorithmic}[1]
\REQUIRE Sparsity $s$, privacy level $\epsilon, \delta$, step-size $\eta$, projection level $C$, sequence of iteration lengths $\{M_\ell\}$, truncation levels $R_\ell$

\FOR{$\ell=0$:} 
    \STATE Play a random arm $a_1\in[K]$ \STATE Observe $r(1) = x_{a_1}(1)^\top \beta^*+\epsilon(1)$. 
    \STATE Set $X_0=\{x_{a_1}(1)\}$, $y_0=\{r(1)\}$.
\ENDFOR
\FOR{$\ell=1,2,\dots$}
    \STATE Set $B=R_\ell+\xmax\bmax$
    \STATE Estimate $\hat{\beta}_\ell = \text{N-IHT}^*(X_{\ell-1}, y_{\ell-1}|M_\ell, R_\ell, B)$
    \STATE Set $X_\ell = \{\}, y_\ell = \{\}$
    \FOR{$t=t_{\ell}, t_{\ell}+1,\dots, t_{\ell+1}-1$}
        \STATE Play $a_t = \argmax_a x_{a}(t)^\top \hat{\beta}_\ell$
        \STATE Observe reward $r(t) = x_{a_t}(t)^\top\beta^* + \epsilon(t)$
        \STATE Store $X_\ell=X_\ell \cup \{x_{a_t}(t)\}$, $y_{\ell}=y_{\ell}\cup\{r(t)\}$
    \ENDFOR
\ENDFOR
\end{algorithmic}
\end{algorithm}

\begin{theorem}
    \label{thm: fliphat DP}
    FLIPHAT preserves $(\epsilon, \delta)$-JDP.
\end{theorem}
To this end, we point out that FLIPHAT also enjoys the reward-DP \citep{hanna2024differentially}, where the algorithm only ensures the privacy protection of the reward stream. 
In this setup, the adversary is only allowed to perturb the reward stream which is quite a strong restriction on the nature of the adversary. JDP protects against a more general adversary and protects the information in both rewards and contexts. 
\subsection{Regret Analysis}
Now, we present the regret upper bound of FLIPHAT. 
We first briefly discuss the assumptions that we impose on the various model components, which are well accepted in the literature on SLCB~\citep{zhang2008sparsity, zhang2010nearly, bastani2020online, li2021regret, chakraborty2023thompson}. \textit{In regret analysis, we do not require any additional assumption for privacy.}


\begin{assumption}[Assumptions on Context Distributions] \label{assumptions: contexts}
We assume that
\begin{enumerate}[label=(\alph*),leftmargin=*,nosep]
\itemsep0em
    \item \label{item: context_bound} \textit{Boundedness:} $\bbP_{x\sim P_i} (\norm{x}_\infty \leq \xmax) =1$ for all $i\in[K]$ for a constant $\xmax\in \bbR^+$.
    \item \label{item: context_subgaussian} \textit{Bounded Orlicz norm:} For all arms $i \in [K]$ the distribution $P_i$ is bounded in the Orlicz norm, i.e., $\norm{X}_{\psi_2}\leq \vartheta$ for $X\sim P_i$ for all $i\in [K]$. 
    \item \label{item: anti-concentration} \textit{Anti-concentration:} There exists a  constant $\xi\in \bbR^+$ such that for each $u\in \{v\in\bbR^d : \norm{v}_0 \leq Cs^*, \norm{v}_2=1\}$ and $h\in \bbR^+$ 
    $P_i(\innerprod{x, u}^2 \leq h) \leq \xi h,$
    for all $i\in [K]$, ,where $C\in (2,\infty)$.
    \item \label{item: sparse_eigenvalue} \textit{Sparsity:} The matrix $\Sigma_i := \bbE_{x\sim P_i} [x x^\top]$ has bounded maximum sparse eigenvalue, i.e., 
    $
\phi_{\max}(Cs^*, \Sigma_i)
        \leq \phi_u   < \infty,
    $
    for all $i \in [K]$,
     where $C$ is the same constant as in part (c).
\end{enumerate}
\end{assumption}

\begin{assumption}[Assumptions on the true parameter]
\label{assumptions: arm-separation}
We assume the following:
\vspace{-5pt}
\begin{enumerate}[label = (\alph*),leftmargin=*,nosep]
\itemsep0em
    \item \label{item: beta-l1} \textit{Sparsity and soft-sparsity:} $\norm{\beta^*}_0 = s^*$ and $\norm{\beta^*}_1 \leq \bmax$ for constants $s^*\in\bbN, \bmax\in\bbR_+$.
    \item \label{item: margin-cond} \textit{Margin condition:} There exists positive constants $\Delta_*, A$ and $\alpha\in [0, \infty)$, such that for $h \in \left[ A \sqrt{\log(d)/T}, \Delta_*\right]$ and $\forall t\in [T]$, $\pr\left( x_{a^*_t}(t)^\top \beta^* \leq \max_{i \neq a^*_t} x_{i}(t)^\top \beta^* +h\right) \leq  \left(\frac{h}{\Delta^*}\right)^\alpha$
\end{enumerate}
\end{assumption}
\begin{assumption}[Assumption on Noise]\label{assumptions: noise}
We assume that the random variables $\{\epsilon(t)\}_{t\in[T]}$ are independent and also independent of the other processes and each one is $\sigma-$Sub-Gaussian, i.e.,  $\bbE [e^{a \epsilon(t)}] \leq e^{\sigma^2 a^2/2}$  for all $ t \in [T]$ and $a\in \bbR$. \vspace*{-.5em}
\end{assumption}

\textit{Remarks:} In ~\ref{assumptions: contexts}\ref{item: context_subgaussian}, $\norm{X}_{\psi_2}:= \sup_{u: \norm{u}_2\le 1} \norm{u^\top X}_{\psi_2}$ is the Orlicz-norm, where for a random-variable $Z$,  $\norm{Z}_{\psi_2} :=\inf \{\lambda >0 : \bbE \exp(Z^2/\lambda^2) \leq 2\}$. It states that the context distribution for each arm is sub-Gaussian, this assumption does not require the distributions to have zero mean. We remark that while we assume i.i.d. contexts across time, for each time $t$, the contexts across different arms are allowed to be correlated. The zero-mean sub-Gaussian assumption on the noise (Assumption~\ref{assumptions: noise}) is satisfied by various families of distributions, including normal distribution and bounded distributions, which are commonly chosen noise distributions. The margin condition (Assumption~\ref{assumptions: arm-separation}\ref{item: margin-cond}) essentially controls the hardness of the bandit instance, with $\alpha=\infty$ making it easiest due to a deterministic gap between arms, and $\alpha=0$ making it hardest, i.e. no apriori information about arm separation. A detailed discussion on the assumptions is in Appendix~\ref{app: assumptions discussion}.

Next, we derive a bound on estimation error of the parameter $\beta^*$ at any episode $\ell$, which is central to the regret analysis and a consequence of Theorem~\ref{thm: general result}. Note that $(X_\ell, y_\ell)$ are only conditionally i.i.d., and not marginally independent. 

\begin{proposition}[Estimation Control for Episode $\ell$]\label{prop: estimation episode}
    Suppose Assumptions \ref{assumptions: contexts}, \ref{assumptions: arm-separation}(a) and \ref{assumptions: noise} holds. If we set $\eta=1/2\bar{\kappa}, C=\bmax, s\gtrsim \kappa^2s^*, M_{\ell+1} = 28\kappa \log (\bmax^2 N_\ell)$ and $R=\xmax\bmax+\sigma\sqrt{2\log N_\ell}$, we have for all $\ell>1$
    \begin{equation}
        \Vert\hat{\beta}_{\ell+1} - \beta^*\Vert_2^2 \lesssim  \Lambda_{\textsc{non-private}} + \Lambda_{\textsc{private}}
    \end{equation}
    with probability at least $1-c_1\exp(-c_2 \log d) - \Pr(\cE_{1\ell}^c) - 1/N_\ell$, where $\Lambda_{\textsc{non-private}} = \sigma^2\kappa^2 \frac{(s + s^*)\log d}{\underline{\kappa}^2 N_\ell}$, and  $\Lambda_{\textsc{private}} = \sigma^2\kappa^2 \frac{(s\log d)^2\log\left( \frac{\log N_\ell}{\delta}\right)\log^3 N_\ell}{N_\ell^2\epsilon^2}$ 
    are the non-private (statistical) and private terms in the estimation bound. Here,  $\cE_{1\ell}=\{\underline{\kappa}\leq \phi_{\min}(Cs^*, \hat{\Sigma}_\ell)<\phi_{\max}(C s^*, \hat{\Sigma}_\ell) \leq \bar{\kappa}\}$, $\kappa=\bar{\kappa}/\underline{\kappa}$, and $\hat{\Sigma}_\ell = X_{\ell} X_{\ell}^\top / N_{\ell}$ is the covariance of the contexts observed in $\ell$-th episode.
\end{proposition}\vspace*{-.5em}


Proposition B.5 of \citep{chakraborty2023thompson} allows us to control the term $\bbP(\cE_{1\ell}^c)$ appearing in the above proposition, which is used to derive the regret guarantee for FLIPHAT, as presented below.
\begin{theorem}[Regret bounds for FLIPHAT]\label{thm: regret}
    Suppose Assumption \ref{assumptions: contexts}, \ref{assumptions: arm-separation}, and \ref{assumptions: noise} hold. Then for any privacy level $\epsilon>0$ and $\delta < e^{-1}$, and with the same choices of the tuning parameters as in Proposition \ref{prop: estimation episode} with $\bar{\kappa} \asymp  \log K$ and $\underline{\kappa} \asymp K^{-1}$, FLIPHAT enjoys the following regret bound 
    \[
    \bbE[R(T)] \lesssim  \xmax \bmax(s + s^*) \log d + I_\alpha \vee J_\alpha(\epsilon, \delta),
    \]\vspace*{-.5em}
    where 
    \begin{align*}\vspace*{-.1em}
    &I_\alpha = 
    \begin{cases}
        \frac{1}{\Delta_*^\alpha}\left(\frac{T^{\frac{1-\alpha}{2} - 1}}{1-\alpha}\right)\{(s + s^*)^2 \log d\}^{\frac{1+\alpha}{2}}, & \alpha \in [0, \infty) \\
         (s + s^*)^2 \log d , & \alpha = \infty. \\
    \end{cases}\\
    &J_\alpha(\epsilon, \delta) = \begin{cases}
        \Psi_\alpha(T)\left\{\frac{(s + s^*)^3\log^2 d \log(1/\delta)}{\epsilon^2}\right\}^{\frac{1+\alpha}{2}}, & \alpha \in [0, \infty)\\
        \frac{(s + s^*)^3 \log^2 d \log(1/\delta)}{\epsilon^2}, & \alpha = \infty.       
    \end{cases}
    \end{align*}
    Here, $\Psi_{\alpha}(T) = \frac{1}{\Delta_*^\alpha}\sum\limits_{n=0}^{\floor{\log T}+1} \frac{n^{2\alpha+2}}{2^{n\alpha}}$.
For $\alpha > 0$, $\Psi_{\alpha}(T) \lesssim \frac{\Gamma(3+2\alpha)}{\Delta_*^{\alpha} (\alpha \log 2)^{3+2\alpha}}$, and $\Psi_{0}(T) \lesssim {\log^3 T}$.\vspace*{-.5em}
\end{theorem}
The first term in the upper bound $\xmax\bmax s^*\log d$ is intuitively the unavoidable part of regret, that is incurred during the initial part of the horizon when $\beta^*$ has not been estimated well enough. The second term is essentially two terms - $I_\alpha$ is the non-private statistical part of the regret while $J_\alpha(\epsilon,\delta)$ is the part arising due to the privacy constraint.
The non-private term $I_\alpha$ is similar to the corresponding term in the regret bound for SLCB in \citep{chakraborty2023thompson, li2021regret}. Furthermore, for $\alpha = 0$, the upper bound in Theorem \ref{thm: regret} matches with the problem-independent lower bound in Theorem \ref{thm:lb_edp} in terms of time horizon $T$ (up to logarithmic factors), context dimension $d$ and privacy parameters $\epsilon,\delta$.
Below, we state the implications.


1. \textit{Dependence on $d, s, s^*,T$}: The term in $J(\epsilon, \delta)$ scaling as $\{(s + s^*)^3\log^2 d\}^{(1+\alpha)/2}$ dominates the non-private term $I_\alpha$ (for fixed $\epsilon, \delta$). The dependence on $T$ is trickier: for the non-private part, for $\alpha=0$, the regret scales as $\sqrt{T}$, for $\alpha=1$ as $\log T$ and the effect fades as $\alpha\to\infty$. For the private part $J_\alpha$, as $\alpha\to 0$, $\Psi_\alpha(T)$ can be upper bounded by $\log^3 T$, while the effect fades as $\alpha$ increases.      

2. \textit{Privacy trade-off for fixed $\alpha$}: For any fixed level of $\alpha$, the regret bound explicates the threshold of the privacy parameter $\epsilon$, under which the cost of privacy dominates the non-private term. For example, with $\alpha=0$ (no margin condition), $J_\alpha$ dominates $I_\alpha$ when $\epsilon/\sqrt{\log(1/\delta)} < \log^3 T\sqrt{\{(s+s^*)\log d\}/T}$. Similarly, for $\alpha=\infty$, the private part dominates when $\epsilon/\sqrt{\log(1/\delta)} < \sqrt{(s+s^*)\log d}$. 
     
3. \textit{Privacy vs. internal hardness}: An interesting feature arising in the reget analysis is the mutual struggle between the hardness coming from the bandit instance's arm-separation (through $\alpha$) and privacy constraint (through $\epsilon,\delta$). To understand this issue, let us denote $\bar{\epsilon}(\alpha)$ to be the threshold for $\epsilon$ under which the regret due to privacy dominates. We note that as $\alpha$ increases, $\bar{\epsilon}(\alpha)$ also increases, while keeping $\delta$ fixed. For example, $\bar{\epsilon}(\infty) = \sqrt{(s+s^*)\log d}\sqrt{\log(1/\delta)}$ and $\bar{\epsilon}(0) = \log^3 T\sqrt{\{(s+s^*)\log d\}\log(1/\delta)/T}$. Thus, when the bandit instance is easy, (high $\alpha$), the hardness due to the privacy is dominant for a larger range of $\epsilon$ (analogously high $\bar{\epsilon}(\alpha)$). Conversely, when the bandit instance is itself hard (low $\alpha$), the regret incurred due to this inherent difficulty quickly crosses the hardness due to privacy. As a result, even for smaller $\epsilon$, regret is dominated by the non-private part.

4. \textit{Achieving minimax optimal bound:} 
The regret bound in Theorem \ref{thm: regret} crucially depends on the quantity $(s+s^*)$. Therefore, the choice of $s$ is very important to avoid the curse of dimensionality. In a high-dimensional sparse setting, $s^*$ is typically very small, i.e., $s^* \ll d $. Therefore, setting $s $ in the order of $o(d)$ is also a valid choice, and an improved dependence on $d$ can be achieved. Moreover, under a constant order knowledge of $s^*$, the tuning parameter $s$ can be chosen as $s \asymp \kappa^2 s^*$ which leads to a near-optimal dependence in $s^*$. In addition, numerical experiments in Appendix~\ref{app: numerical s} show that the regret only scales polynomially in $s >s^*$ as shown in Theorem \ref{thm: regret}.

\vspace*{-.5em}\section{NUMERICAL EXPERIMENTS}\vspace*{-.5em}
We present an empirical study of the performance of FLIPHAT through several numerical experiments. All experiments were run locally on Macbook Pro with 16 GB RAM and 1.4 GHz Quad-Core Intel Core i5. 
Both of the experiments are setup in the problem-independent setting $\alpha=0$, and the plots show both the mean and $95\%$-CI for the mean (as colored bands) over 60 repetitions of each experiment. For space constraint, we present two of these experiments here, and postpone the additional details and studies to the Appendix \ref{sec: numerical details appendix}. 

In the first experiment, \textbf{we test the time evolution of regret of FLIPHAT}, and illustrate the regret against time $T$ in Figure \ref{fig: regret}(Left) in the case where $d=400, s^*=5, K=3$ up to $T=20000$ for different choices of privacy parameter $\epsilon\in\{0.5,1,2,5,10\}, \delta=0.01$. We also include \textit{random bandit} (taking random action at each time), which corresponds to $\epsilon=0$ (full privacy, no effective utility), and also the sparsity agnostic bandit \citep{oh2021sparsity}, which corresponds to $\epsilon\to\infty$, i.e. without any privacy guarantee, as the baselines. The contexts are drawn i.i.d. from $\cN(0,\Sigma)$, where $\Sigma_{ij}=0.1^{|i-j|}$ (autoregressive design), with observation noise being a centered Gaussian with $\sigma=0.1$. For the second experiment, \textbf{we illustrate the effect of context dimension $d$ on the regret suffered by FLIPHAT} at $T=10000$ for 4 different choices of $\epsilon$ (Figure \ref{fig: regret}-Right) in a setting similar to the first experiment, while $d$ takes 12 equi-spaced values from 400 to 4000.

\textbf{Results.} Figure~\ref{fig: regret} (Left) demonstrates the sublinear behavior of FLIPHAT's regret that approaches the sparsity agnostic bandit's regret for larger $\epsilon$. Figure~\ref{fig: regret} (Right) shows that $R(T)$ with $T=10000$ is nearly linear in $\log d$, where the slope $\beta$ (in the legend) is computed by fitting a linear model. This matches with the theoretical results in Theorem~\ref{thm: regret}, i.e., the cost of privacy in FLIPHAT scales as $O(\log d)$. 

\begin{figure}[t!]
\centering  \vspace*{-.5em}
  \includegraphics[width=\columnwidth]{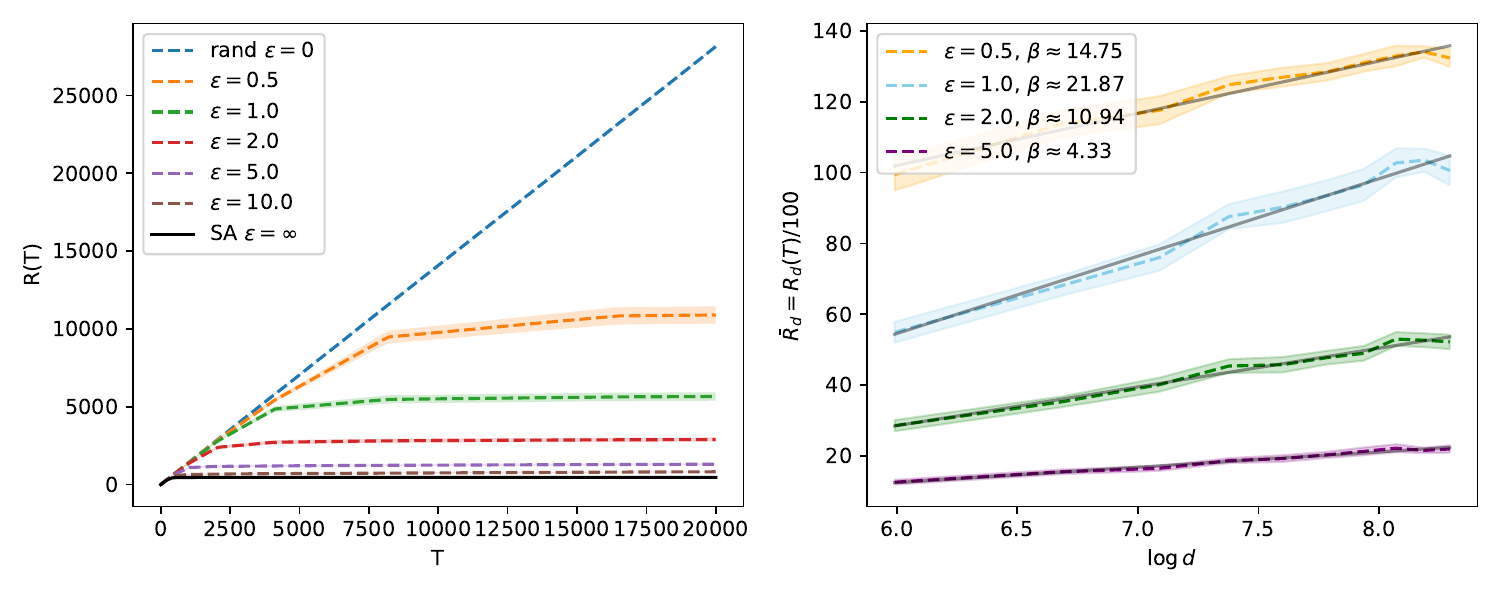}\vspace*{-2em}
  \caption{(Left) Regret vs $T$ for different privacy level $\epsilon$, (Right) Regret at $T=10000$ vs $\log(d)$ for different $\epsilon$}\label{fig: regret}
\end{figure}
 
 
\vspace*{-.5em}\section{CONCLUSION}\vspace*{-.5em}
In this paper, we study the SLCBs under JDP constraints. 
We establish two minimax regret lower bounds under JDP settings, demonstrating phase transition between the hardness of the problem depending on the privacy level.
We propose FLIPHAT, which is the first algorithm to achieves sharp regret-privacy trade-off in SLCBs. FLIPHAT is based on the N-IHT algorithm, for which we propose a refined convergence analysis under a general setting than the existing literature.


An interesting direction would be to study FLIPHAT under an adversarial sparse LCB setting. It would also be an interesting future study to explore the lower bound using the margin condition and obtain bounds dependent on the parameter $\alpha$ in the margin condition.


%% file: arxiv_appendix.tex
\appendix
\onecolumn
\section{\uppercase{Background on differential privacy}}\label{section: background_privacy}
In this section we will formally define the notion of differential privacy. On a high level, differential privacy requires the output of a randomized procedure to be robust against small perturbation in the input dataset, i.e., an attacker can hardly recover the presence or absence of an individual in the dataset based on the output only. 

Consider a dateset $D = \{z_1, \ldots, z_n\} \in \cZ^n$ consisting of $n$ data points in the sample space $\cZ$. A \emph{randomized} algorithm $\cA$ takes the dataset $D$ as input and maps it to $\cA(D) \in \cO$, an output space. Two datasets $D$ and $D^\prime$ are said to be \emph{neighbors} if they differ only in one entry.

\begin{definition}[\cite{dwork2006differential}]
    \label{def: DP}
    Given the privacy parameters $(\epsilon, \delta)$, a randomized algorithm is said to satisfy the $(\epsilon, \delta)$-DP if 
    \[
    \pr (\cA(D) \in \cO) \le e^\epsilon\pr(\cA(D^\prime) \in \cO) + \delta
    \]
    for any measurable set $\cO \in range(\cA)$ and any neighboring datasets $D$ and $D^\prime$.
    
\end{definition}
The probability in the above definition is only with respect to the randomization in $\cA$, and it does not impose any condition on the distribution of $D$ or $D^\prime$. For small choices of $\epsilon$ and $\delta$, Definition \ref{def: DP} essentially says that the distribution of $\cA(D)$ and $\cA(D^\prime)$ are nearly indistinguishable form each other for al choices of $D$ and $D^\prime$. This guarantees
strong privacy against an attacker by masking the presence or absence of a particular individual
in the dataset. 

There are several privacy preserving mechanisms that are often the building blocks of other complicated DP mechanisms. A few popular examples include the Laplace mechanism \citep{dwork2006calibrating}, Gaussian mechanism \citep{dwork2006our}, and Exponential mechanism \citep{mcsherry2007mechanism}. We only provide the details of the first technique, since the
other two techniques are out-of-scope for the methods and
experiments in this paper.

\paragraph{Laplace mechanism:} Let $f: \cZ^n \to \bbR^d$ be a deterministic map. The \emph{Laplace mechanism} preserved privacy by perturbing the output $f(D)$ with noise generated from the Laplace density $\Lap(\lambda)$, whose density is $(2 \lambda)^{-1} \exp( - \abs{x}/ \lambda)$. The scaling parameter $\lambda$ is chosen based on the sensitivity of the function $f$, which is defined as 
\[
\Delta f := \sup_{D, D^\prime : \text{$D, D^\prime$ are neighbors}}\norm{f(D) - f(D^\prime)}_1.
\]

In particular, for any deterministic map $f$ with $\Delta f< \infty$, the mechanism that, on database $D$, adds independent draws from $\Lap(\Delta f/\epsilon)$ to each of the $m$ components of $f(D)$, is $(\epsilon, 0)$-DP.

The Laplace mechanism is also an important building block of the well-known \emph{peeling} algorithm (Algorithm \ref{alg: peeling}) which is used for private top-$s$ selection problem \citep{dwork2021differentially}.
We denote the output of peeling applied to a vector $v \in \bbR^d$ as $\cP_s(v ;\epsilon,\delta, \lambda)$, where $\epsilon, \delta$ are privacy parameters and $\lambda$ is the Laplace noise scale parameter. For $v\in \bbR^d$, and an index set $S\subseteq[d]$, $v_S\in\bbR^d$ denotes the vector such that $v_S$ matches $v$ on indices in $S$ and is 0 elsewhere. Therefore, peeling returns a $s$-sparse vector in $\bbR^d$. 
\begin{algorithm}
\caption{Peeling \citep{dwork2021differentially}}\label{alg: peeling}
\begin{algorithmic}[1]
\REQUIRE Vector $v\in\bbR^d$, sparsity $s$, privacy level $\epsilon, \delta$, Laplace noise scale $\lambda$

\STATE \textbf{Initialize}: $S=\emptyset$ and set $\xi = \lambda \frac{2\sqrt{3s\log (1/\delta)}}{\epsilon}$
\FOR{$i = 1, 2,\dots, s$:}
    \STATE Generate $\bw_i = (w_{i1},\dots,w_{id})\overset{i.i.d.}{\sim}\Lap(\xi)$ 
    \STATE $j_i^* = \argmax_{j\in [d]\setminus S} |v_j| + w_{ij}$
    \STATE Update $S\leftarrow S \cup \{j_i^*\}$
\ENDFOR
\STATE Set $\tilde{P}_s(v) = v_{S}$
\STATE Generate $\tilde{\bw} = (\tilde{w}_1,\dots,\tilde{w}_d)\overset{i.i.d.}{\sim} \Lap(\xi)$
\STATE \textbf{Return} $\cP_s(v; \epsilon, \delta, \lambda):=\tilde{P}_s(v) + \tilde{\bw}_S$
\end{algorithmic}
\end{algorithm}

The following lemma shows that peeling is $(\epsilon, \delta)$-DP.
\begin{lemma}[\cite{dwork2021differentially}]
    \label{lemma: peeling DP}
    If for every pair of adjacent datasets $D, D^\prime$, we have $\norm{v(D) - v(D^\prime)}_\infty \le \lambda$, then \textit{Peeling}  \ref{alg: peeling} is $(\epsilon, \delta)$-DP. 
\end{lemma}

Lemma \ref{lemma: peeling DP} gives the privacy guarantees for the Peeling algorithm as described above. We note that two sources of noise are injected through the peeling process -- $\bw_i$ during the \textit{for loop}, and $\tilde{\bw}$ at the end. 

\section{\uppercase{Discussion on Assumptions}}\label{app: assumptions discussion}
Assumption \ref{assumptions: contexts}\ref{item: context_bound} about bounded contexts are standard in the bandit literature to obtain regret bounds independent of the scaling of the contexts. Assumption \ref{assumptions: contexts}\ref{item: context_subgaussian} says that the context distribution for each arm is a sub-Gaussian distribution with parameter $\vartheta$, this assumption does not require the distributions to have zero mean. This is indeed a very mild assumption on the context distribution and a broad class of distributions, such as truncated multivariate Gaussian and uniform on a compact set, enjoys such property. In comparison, several works in the literature on sparse contextual bandits \citep{oh2021sparsity, kim2019doubly, li2021regret} assume bounded $L_2$ norm, i.e. $\norm{X}_2\leq L$, which in the high-dimensional setting is much stronger. Indeed $\norm{X}_2\leq L$ implies $\norm{X}_{\psi_2}\leq (L/\log 2)^{1/2}$ and in the example of truncated normal, $\norm{X}_2 = \Theta(\sqrt{d})$.
 Assumption \ref{assumptions: contexts}\ref{item: anti-concentration} is an anti-concentration condition that ensures that the directions of the arms are well spread across every direction, thereby assisting in the estimation accuracy for $\beta^*$. Recent works \citep{li2021regret,chakraborty2023thompson} use similar assumptions to establish sharp estimation rates for $\beta^*$ under bandit settings.
 Assumption \ref{assumptions: contexts}\ref{item: sparse_eigenvalue} imposes an upper bound on the maximum sparse eigenvalue of $\Sigma_i$ which is a common assumption in high-dimensional literature \citep{zhang2008sparsity,zhang2010nearly}. Below we define the minimum and maximum sparse eigenvalue conditions formally.
 \begin{definition}[Sparse Riesz Condition]
\label{def: sparse eigen value}
For a $d\times d$ positive semi-definite matrix $A$, its maximum and minimum sparse eigenvalues with parameter $s\in [d]$ are defined as 
\begin{align}
    \phi_{\min}(s; A):= \inf_{u: u \neq 0, \norm{u}_0 \leq s} u^\top A u /\norm{u}_2^2, \quad \phi_{\max}(s; A) := \sup_{u: u \neq 0, \norm{u}_0 \leq s} u^\top A u/\norm{u}_2^2\,.
\end{align}
\end{definition}
 
For the assumptions on the parameter, boundedness of $\beta^*$ ensures that the final regret bound is scale free and is also a standard assumption in the bandit literature \citep{bastani2020online, abbasi2011improved, chakraborty2023thompson}.

The margin condition in the second part essentially controls the probability of the optimal arm falling into an $h$-neighborhood of the sub-optimal arms. As $\alpha$ increases, the margin condition becomes stronger. As an illustration, consider the two extreme cases $\alpha = \infty$ and $\alpha = 0$. The $\alpha = \infty$ case enforces a deterministic gap between rewards corresponding to the optimal arm and sub-optimal arms. This is the same as the ``gap assumption'' in \cite{abbasi2011improved}. Thus, it is easy for any bandit policy to recognize the optimal arm, which is reflected in the poly-logarithmic dependence on horizon $T$ in the regret bound of Theorem 5 in \cite{abbasi2011improved}. In contrast, $\alpha = 0$ corresponds to the case when there is no apriori information about the separation between the arms, and as a consequence, we pay the price in regret bound by a $\sqrt{T}$ term \citep{hao2021information, agrawal2013thompson, chu2011contextual}. Assumption \ref{assumptions: arm-separation}\ref{item: margin-cond} is also considered in \cite{li2021regret, chakraborty2023thompson}, and more details can be found therein.

\section{\uppercase{Regret Lower Bound: Details from Section}~\ref{sec: lower bound}}
\label{app: lower bound proof new}
Before getting into the proof, we discuss some additional details about the probability structure in the bandit environment, due to the randomness coming from the (i) bandit instance (through contexts and rewards) and (ii) policy (randomness due to privacy). 

Let $\cB_{\beta^*}$ denote the space of joint distributions $P_{\bX,\br}$ of contexts and rewards up to time $T$ when the underlying parameter is $\beta^*$. In particular, 
\begin{align*}
    P_{\bX,\br} = \otimes_{t\in[T]} P_{\bX_t, \br_t}.
\end{align*}
Here, $P_{\bX_t, \br_t}$ is the joint distribution of the contexts $(x_1(t),\dots,x_K(t))\in\bbR^{Kd}$ and rewards $\br_t=(r_1(t),\dots,r_K(t))\in\bbR^K$ at time $t$, induced through $(x_1(t),\dots,x_K(t))\sim P_{\bX}$ (with the marginal for $x_i(t)$ being $P_i$) and conditional on $x_i(t)$, $r_i(t)=x_i(t)^\top\beta^* + \epsilon_i(t)$. $\epsilon_i(t)\sim P_{\epsilon}$ independently for all $i,t$. A particular choice of $\beta^*, P_{\bX}$ and $P_{\epsilon}$ lead to a particular bandit instance $P_{\bX,\br}\in \cB_{\beta^*}$. { {To control for the scaling in the contexts and the conditional rewards, we assume that for $(x_1,\dots,x_K)\sim P_{\bX}$, the marginal covariance of each $x_k$ has eigenvalues bounded by $\sigma_{X,\min}^2$ and $\sigma_{X,\max}^2$. We also assume that $P_{\epsilon}$ is a sub-Gaussian distribution with variance-proxy $\sigma^2$. The dependence on these parameters will be denoted by $\cB_{\beta^*,\sigma_X,\sigma}$ in the space of joint distributions of contexts and rewards.}}. However, given a particular bandit instance, the learner does not get to observe all this data. This is where the policy comes in.

A policy $\pi$ is defined by a sequence of randomized decisions $a^{\pi}_t:\bH_{t-1}^{\pi}\times \bX_t \to \Delta([K])$, where $\bH_{t-1}^{\pi}=\{x_{a^\pi_s}(s), r_{a_s^\pi}(s): 1\leq s <t\}$ is the observed history and $\Delta([K])$ is the space of probability distributions on $[K]$ (the action set). When the policy is clear from context, it is often abbreviated to saying that $a_t$ is the (random) \textit{action} taken at time $t$, depending on the history (which also depends on the policy-induced actions at previous time points) and the current contexts. The stochasticity in $a_t^{\pi}$ could be from the randomness required for exploration, or privacy concerns. The expectation in the expected regret $\bbE [R_{\pi}(T)]$ refers to the expectation taking with respect to the randomness arising due to the interaction of the bandit instance $P_{\bX,\br}$ and policy $\pi$ and the minimax regret for JDP-SLCB problem is
$$R^{\mathrm{minimax}}_{(\epsilon,\delta)}(T) = \inf_{\pi\in \Pi(\epsilon,\delta)} \sup_{P_{\bX,\br}\in \cB_{\sigma_X,\sigma}} \bbE[R_{\pi}(T)]$$
where the infimum is with respect to all $(\epsilon,\delta)-$JDP policies $\pi$ and the supremum is with respect to $\cB_{\sigma_X,\sigma}=\cup \{\cB_{\beta^*,\sigma_X,\sigma}: \beta^* \text{ satisfies Assumption~\ref{assumptions: arm-separation}\ref{item: beta-l1}}\}$, i.e. all SLCB instances with the unknown parameter having sparsity $s^*$. { {We denote the regret minimax bound for $\epsilon-$JDP policy by $R^{\text{minimax}}_{\epsilon}$, where the infimum is taken over all $\epsilon-$JDP policies. Note that under this definition, the minimax regret also depends on the noise scaling $\sigma^2$ and the contexts' marginal scaling controlled by $\sigma_X^2$. To ensure that $\bbE\norm{x}_{\infty} \asymp 1$, we take $\sigma_X=1/\sqrt{\log d}$.}}

\paragraph{Step 1: Choosing Hard Instances}

Now, we choose several hard instances $P_{\bX,\br}\in \cB_{\sigma_X,\sigma}$ as follows -- consider a 2-armed SLCB (i.e., $K=2$), with $\beta^*\in \tilde{\bB}=\{\beta\in\bbR^d\mid \beta_i\in \{-r,0,r\}, \norm{\beta}_0=s^*\}$ where $s^*$ is the assumed true sparsity level and $r>0$ will be specified later. Then $|\tilde{\bB}|={{d}\choose{s^*}}2^{s^*}$. Let
$$\delta(t) = \{\delta \mid \norm{\beta-\beta'}_2\geq \delta, \text{ for all } \beta,\beta'\in \tilde{\bB} \text{ s.t. } d_{H}(\beta,\beta')> t\},$$
where $d_H$ is the Hamming distance. Considering a $t-$packing of $\tilde{\bB}$ (in the Hamming distance) following \cite{duchi2013distance}, this leads to a $\delta(t)-$packing in the $\ell_2$ distance. In our case, with the given construction of $\tilde{\bB}$, this leads to $\delta(t)>\max\{1, \sqrt{t}\}r$. Let $\bB^*=\{\beta_1,\dots,\beta_M\}$ be the elements in this $\delta(t)$ packing with $t=\ceil{s^*/4}$, where $\log M\geq cs^*\log (d/s^*)$ for some numerical constant $c$. Moreover, for this choice of $t$, we have for any $\beta,\beta'\in\bB^*$, $\norm{\beta-\beta'}^2_2 \geq \max\{1,s^*/4\} r^2 =:\alpha$ (by construction of the packing). For the contexts, we consider $P_{\bX}=\cN(\boldsymbol{0}, \sigma_X^2I_d)\otimes \cN(\boldsymbol{0}, \sigma_X^2I_d)$ (independent multivariate standard normal for both arms). Let $z_t=x_1(t)-x_2(t)$ for each $t$. Lastly, let $\nu$ denote the uniform distribution on $\bB^*$.

\paragraph{Step 2: Bounds on distances between context-rewards under different true parameters}

Denote the space of bandit instances using $\beta^*\in \bB^*$ and contexts as constructed above as $\tilde{\cB}$ and for each $\beta\in \tilde{\cB}$, let $p_t(\cdot|\beta)$ denote the distribution of \textit{all} the contexts and rewards till time $t$ under true parameter $\beta$. We first establish upper bounds between the KL divergence and total variation distances between $p_t(\cdot|\beta)$ and $p_t(\cdot|\beta')$ for $\beta,\beta'\in\bB^*$. These upper bounds will be used while applying differentially private Fano's inequalities. 

 On the other hand, the data consists of $\{X_j, r_j\}$ for $j\in [Kt]$ { {(recall that the data consists of all rewards and contexts till time $t$ irrespective of the previously chosen actions -- hence, the data consists of i.i.d. observations)}}. The KL can be bounded as
\begin{align*}
    \KL(p(\cdot|\beta)\Vert p(\cdot|\beta')) &= Kt \,\bbE_{x\sim N(0, \sigma_X^2 I)}\KL\left(\cN(x^\top\beta, \sigma^2) \Vert \cN(x^\top\beta', \sigma^2\right) \\
    &= Kt \bbE_{x\sim N(0, \sigma_X^2 I)} \left[\frac{(x^\top(\beta-\beta'))^2}{2\sigma^2}\right] \\
    &= \frac{Kt \sigma_X^2 \norm{\beta-\beta'}_2^2}{2\sigma^2} \leq \frac{2Kt\sigma_X^2 s^* r^2}{\sigma^2}=:\gamma.
\end{align*}
The last inequality arises since the worst distance can occur if $\beta,\beta'$ has the same support but opposing signs (then $\norm{\beta-\beta'}_2^2 = 4s^*r^2$).

{ {Finally, the total variation distance can be bounded using Pinsker and Jensen's inequality as follows
\begin{align*}
    \TV(p(\cdot|\beta), p(\cdot|\beta')) &\leq Kt \bbE_x\TV(\cN(x^\top\beta,\sigma^2 ), \cN(x^\top\beta', \sigma^2)) \\
    &\leq Kt \sqrt{\frac{1}{2} \bbE_x \KL(\cN(x^\top\beta,\sigma^2 ), \cN(x^\top\beta', \sigma^2))} \\
    &\leq Kt \frac{\sigma_X}{\sigma}\sqrt{s^*}r =: D.
\end{align*}}}

\paragraph{Step 3: Reduction to estimation lower bound}
Given a bandit instance from the above class, suppose $\pi$ is any bandit policy. Define the $\pi-$induced estimator in period $t$, $\beta^{\pi}_t$ to be the maximizer of
\begin{align}\label{eq: maximizer beta pi hat}
    f^{\pi}_t(\beta):=P(a^{\pi}_t=1, z_t^\top \beta>0|\bH^{\pi}_{t-1}) + P(a^{\pi}_t=2, z_t^\top \beta\leq 0|\bH^{\pi}_{t-1})
\end{align}
where the maximization is taken over all DP (random) maps $\beta:\bH^{\pi}_{t-1}\to \cR(\bbR^d)$, the space of all mappings of the history till time $t-1$ to a random variable taking values in $\bbR^d$ that are $(\epsilon,\delta)-$DP with respect to the history. Let $\hat{\Theta}_t^{\epsilon, \delta}$ denote this space of DP random maps  { {(the notion of DP could also be $\epsilon-$DP, for which we denote the space of DP maps as $\hat{\Theta}_t^{\epsilon}$)}}. Note that the probability $P$ in the above display is with respect to the random action $a^{\pi}_t$ and the (random) estimator $\beta$. Now, since we are only interested in JDP algorithms, let $\pi$ be any $(\epsilon,\delta)-$JDP bandit policy. We note that the degenerate maps $\beta^*:\bH^{\pi}_t \mapsto \delta_{\beta^*}$ (abusing notation slightly, viewing $\beta^*\in\bbR^d$ equivalent to the map which takes any history to the degenerate distribution at $\beta^*$) are $(\epsilon,\delta)-$DP with respect to the data $\bH^{\pi}_t$ for any $\beta^*$ --hence, for any $\beta^*\in\bB^*$, the corresponding map $\beta^*\in \hat{\Theta}_t^{\epsilon, \delta}$. Since $\beta^{\pi}_t$ and $\beta^*$ are both in $\hat{\Theta}_t^{\epsilon, \delta}$, Lemma 3.1 in \citep{he2022reduction} applies and consequently, Theorem 3.2 in \citep{he2022reduction} can be used to conclude that
\begin{align}
    \inf_{\pi\in\Pi}\sup_{P_{\bX,\br}\in\cB}\bbE[R_\pi(T)] \geq \inf_{\pi\in\Pi}\sup_{P_{\bX,\br}\in \tilde{\cB}} \bbE[R_{\pi}(T)] =\Omega\left(\frac{1}{r\sqrt{s^*}}\sum_{t=1}^T \inf_{\beta_t\in \hat{\Theta}_t^{\epsilon, \delta}} \bbE_{\nu}\left[\norm{\beta_t - \beta^*}_2^2\right]\right).
\end{align}
In the above, $\bbE_{\nu}[\norm{\beta_t-\beta^*}_2^2]$ is the Bayesian 2-norm risk, where the expectation is taken jointly with respect to the randomness in both $\beta_t$ and $\beta^*\sim \nu$. We remark that this technique does not require a policy to have a specific form that depends on DP estimators sequentially.

Now, we denote by $\Theta_t^{\epsilon, \delta}$ the collection of $(\epsilon, \delta)$-DP estimators with respect oto $\bH_t:=\{x_a(s), r_a(s): s\in [t], a\in\{1,2\}\}$. Therefore, trivially we have $\bH_t^\pi \subseteq \bH_t$. Also, quite crucially, for an $\beta_t \in \hat{\Theta}_t^{\epsilon, \delta}$, we know that it sis only dependent on $\bH_t^\pi$ and is also $(\epsilon,\delta)$-DP with respect to $\bH_t^\pi$. This suggests that $\beta_t$ is also $(\epsilon,\delta)$-DP with respect to the input data $\bH_t$. To see this, let $\bH_t$ and $\bH_t^\prime$ be two neighboring datasets. There are only two possibilities in this case: (i) $\bH_t^\pi = \bH_t^{\prime \pi}$, or (ii) $d_{H}(\bH_t^\pi, \bH_t^{\prime \pi}) = 2$.
\begin{itemize}
    \item \textbf{Case (i):} Since $\beta_t$ only depends on $\bH_t^\pi$ (or $\bH_t^{\prime\pi}$), we can conclude
    \begin{align*}
        \pr(\beta(\bH_t) \in O)&= \pr(\beta_t(\bH_t^\pi) \in O)\\
        & = \pr(\beta_t(\bH_t^{\prime \pi})\in O)\\
        & = \pr(\beta_t(\bH_t^{\prime})\in O).
    \end{align*}

    \item \textbf{Case (ii):} In this case $\bH_t$ and $\bH_t^{\prime \pi}$ are neighboring datasets. Therefore, we have
     \begin{align*}
        \pr(\beta(\bH_t) \in O)&= \pr(\beta_t(\bH_t^\pi) \in O)\\
        & \le e^\epsilon\pr(\beta_t(\bH_t^{\prime \pi})\in O) + \delta\\
        & = e^\epsilon\pr(\beta_t(\bH_t^{\prime})\in O) + \delta.
    \end{align*}
\end{itemize}
The above argument essentially shows that $\hat{\Theta}_t^\pi \subseteq \Theta_t$. This immediately shows that 
\begin{align*}
\inf_{\beta_t \in \hat{\Theta}_t^{\epsilon, \delta}} \bbE_\nu [\norm{\beta_t - \beta^*}_2^2] &\ge \inf_{\beta_t \in \Theta_t^{\epsilon, \delta}} \bbE_\nu [\norm{\beta_t - \beta^*}_2^2].
\end{align*}

{  {Now, we divide the last part into 2 steps - for the first, we prove lower bound for $\epsilon-$JDP policies using the $\epsilon-$DP Fano's inequality in \cite{acharya2021differentially}, while for the second part, we use the $\rho-$zCDP  Fano's inequality in \cite{kamath2022improved} and exploit the connection of $\rho-$zCDP with $(\epsilon,\delta)-$DP \citep{Bun:tCDP} to establish regret lower bound for $(\epsilon,\delta)-$JDP policy.}}

\paragraph{Step 4(a): Analysis under $\epsilon-$DP}
Following the same steps as in the proof of Theorem 2 \cite{acharya2021differentially}  we get the following:
\paragraph{non-private lower bound:}
\begin{align*}
    \bbE_{\nu}[\norm{\beta_t - \beta^*}_2^2] &\geq \frac{\max\{1,s^*/4\} r^2}{2}\left(1 - \frac{\gamma + \log 2}{\log M}\right)\\
    &= \frac{\max\{1,s^*/4\} r^2}{2}\left(1 - \frac{\frac{4t\sigma_X^2s^* r^2}{\sigma^2} + \log 2}{cs^* \log (d/s^*)}\right).
\end{align*}
\paragraph{private lower bound:}
\begin{align*}
    \bbE_{\nu}[\norm{\beta_t-\beta^*}_2^2] &\geq 0.4\alpha \min\left\{1, \frac{M}{e^{10\epsilon D}}\right\} \\
    &\geq 0.4 \max\{1,s^*/4\} r^2\min\left\{1, \exp\left(cs^*\log (d/s^*)-10\epsilon Kt\frac{\sigma_X}{\sigma}\sqrt{s^*}r\right)\right\}.
\end{align*}
For $s^*>4$ and $\epsilon>0$ small (so that the minimum term is the exponential one), we have
\begin{equation}
    \bbE_{\nu}[\norm{\beta_t-\beta^*}_2^2] \geq \frac{s^*}{4}r^2 \max\left\{0.5\left(1 - \frac{\frac{2Kt\sigma_X^2s^* r^2}{\sigma^2} + \log 2}{cs^* \log (d/s^*)}\right), 0.4 \exp\left(cs^*\log (d/s^*)-10\epsilon Kt\frac{\sigma_X}{\sigma}\sqrt{s^*}r\right)\right\}.
\end{equation}
Hence the minimax regret is (noting $K=2)$
\[
\inf_{\pi \in \Pi} \sup_{P_{\bX, \br} \in \cB} \bbE [R_\pi(T)] \gtrsim \frac{r \sqrt{s^*} T}{4} \max\left\{ 1 - \frac{\frac{4T\sigma_X^2s^* r^2}{\sigma^2} + \log 2}{cs^* \log(d/s^*)}, \exp \left(cs^* \log(d/s^*) - 20 \epsilon T \frac{\sigma_X}{\sigma} \sqrt{s^*}r\right)\right\}
\]
If we set $r^2 = \frac{\sigma^2\log(d/s^*)}{8T}$, then the non-private lower bound becomes
\[
\inf_{\pi \in \Pi} \sup_{P_{\bX, \br} \in \cB} \bbE [R_\pi(T)] \gtrsim \sigma \sqrt{s^* T \log(d/s^*)} \left(1 - \frac{\sigma_X^2}{2} - \frac{\log 2}{cs^* \log(d/s^*)}\right).
\]
Now. recall that $\sigma_X^2 = 1/\log d$. Therefore, whenever $d>9$ and $s^* \log(d/s^*) > 4c^{-1} \log 2$, the above lower bound becomes $\Omega(\sqrt{s^* T \log(d/s^*)})$.

Now, we shift focus to the private part of the lower bound. We set $r = \frac{c \sigma\sqrt{s^*} \log(d/s^*)}{20 \epsilon \sigma_X T}$, which yields 
\[
\inf_{\pi \in \Pi} \sup_{P_{\bX, \br} \in \cB} \bbE [R_\pi(T)] \gtrsim \sigma \frac{s^* \log(d/s^*)\log^{1/2}(d) }{\epsilon} \gtrsim \sigma \frac{\log^{3/2}(d/s^*) }{\epsilon}.
\]
Combining the above facts, we finally have
$$R^{\mathrm{minimax}}_{\epsilon}(T) = \Omega\Big(\max \lbrace s^* \log^{3/2} (d/s^*)\epsilon^{-1}, 
    {\sqrt{s^* T\log (d/s^*)} }\rbrace\Big).$$
    
\paragraph{Step 4(b): Analysis under $(\epsilon,\delta)-$DP}  For $(\epsilon,\delta)-$DP regret lower bound, we use the lower bound technique for $\rho-$zCDP \citep{kamath2022improved} and its connection to $(\epsilon,\delta)-$DP \citep{bun2016concentrated}. Towards this, given $(\epsilon,\delta)$, we construct $\rho:=\rho(\epsilon,\delta)$ such that any $(\epsilon,\delta)-$DP algorithm is $\rho-$zCDP -- thus, $\Theta_t^{\epsilon,\delta}\subseteq \Theta_t^{\rho}$, where the latter is the space of all $\rho-$zCDP estimators constructed using the entire history $\bH_{t-1}$. By \citep[Lemma 3.7]{bun2016concentrated}, we know that if $(\epsilon,\delta)$ satisfies $\epsilon=\xi + \sqrt{\rho_0\log (1/\delta)}$, then any $(\epsilon,\delta)-$DP mechanism is also $(\xi-\rho_0/4+5\rho_0^{1/4}, \rho_0/4)-$zCDP. Hence, by choosing $\xi=\epsilon-\sqrt{\rho_0\log(1/\delta)}$, we note that if $\rho_0$ satisfies

$$\epsilon - \sqrt{\rho_0\log(1/\delta)} -\rho_0/4 + 5 \rho_0^{1/4} = 0$$

then, the mechanism is $\rho$-zCDP with $\rho = \rho_0/4$. Furthermore, we can show that a positive root of the above quartic equation satisfies $\rho_0\leq 20^{4/3}\epsilon^2/\log(1/\delta)$ (for the fixed $\epsilon,\delta$). Hence, with this choice of $\rho = \rho_0/4$,
\begin{align*}
    \inf_{\beta_t \in \Theta_t^{\epsilon, \delta}} \bbE_\nu [\norm{\beta_t - \beta^*}_2^2] &\ge \inf_{\beta_t \in \Theta_t^{\rho}} \bbE_\nu [\norm{\beta_t - \beta^*}_2^2].
\end{align*}

Now, to lower bound the second term, we use \cite[Theorem 1.4]{kamath2022improved}. The non-private version is the same as in the $\epsilon-$DP case discussed above and we do not reproduce it. We focus on the private case. Here the loss is $\ell(x,y)=\norm{x-y}^2$, under which we have, by arguments above, $\ell(\beta_,\beta') \geq \max\{1, s^*/4\}r^2$ for any $\beta,\beta'\in \bB^*$. Furthermore, $n=2t$ is the number of i.i.d. observations used in the above estimation problem and $\alpha = \bbE_x\TV(\cN(x^\top\beta,\sigma^2),\cN(x^\top\beta',\sigma^2))\leq r\sigma_X\sqrt{s^*}/\sigma$. Hence, we have
\begin{align*}
    \bbE_{\nu} [\norm{\beta_t-\beta^*}_2^2] &\geq \frac{\max\{1,s^*/4\}r^2}{2}\left(1 - \frac{\rho\left(\frac{4t^2\sigma_X^2 s^* r^2}{\sigma^2} + 2t\frac{r\sigma_X\sqrt{s^*}}{\sigma}\left(1 - \frac{r\sigma_X\sqrt{s^*}}{\sigma}\right)\right) + \log 2}{cs^*\log (d/s^*)}\right) \\
    &\geq \frac{\max\{1,s^*/4\}r^2}{2}\left(1 - \frac{\frac{4\rho t^2\sigma_X^2 s^* r^2}{\sigma^2} + \frac{2\rho t \sigma_X r\sqrt{s^*}}{\sigma} + \log 2}{cs^*\log (d/s^*)}\right).
\end{align*}
Thus for $s^*>4$, the minimax regret takes the form
\begin{align*}
    \sup_{P_{\bX,r}} \bbE[R_{\pi}(T)] \geq \frac{r\sqrt{s^*} T}{8}\left(1 - \frac{\frac{4\rho\sigma_X^2}{\sigma^2}s^*r^2T^2 + \frac{2\rho \sigma_X}{\sigma}\sqrt{s^*} r T + \log 2}{cs^* \log(d/s^*)}\right).
\end{align*}
Putting $r=\sqrt{\frac{c\sigma^2\log(d/s^*)}{16\rho\sigma_X^2 T^2}}$, we have
\begin{align*}
    \inf_{\pi} \sup_{P_{\bX,r}}\bbE[R_{\pi}(T)] \gtrsim \frac{\sigma\sqrt{s^*\log(d/s^*)}}{\sigma_X\sqrt{\rho}}\left(1 - \frac{1}{4} - \frac{\sqrt{\rho}}{2\sqrt{c s^*\log(d/s^*)}} - \frac{\log 2}{cs^*\log(d/s^*)}\right).
\end{align*}
We put $\sigma_X=1/\sqrt{\log d}$ as before. Recall that $\rho \leq 20^{4/3}\epsilon^2/\log(1/\delta)$. Thus, whenever $s^*\log(d/s^*) \geq \max\{\frac{\log 2}{8c}, \frac{20^{2/3}\epsilon}{16\sqrt{\log(1/\delta)}}\}$, we have
\begin{align*}
    \inf_{\pi} \sup_{P_{\bX,r}}\bbE[R_{\pi}(T)] \gtrsim \frac{\sigma \sqrt{s^* \log(d/s^*)\log d \log(1/\delta)}}{\epsilon}\geq \frac{\sigma \log(d/s^*)\sqrt{s^*\log(1/\delta)}}{\epsilon}.
\end{align*}

Combining the non-private and private parts, we get
\begin{align*}
    R^{\text{minimax}}_{\epsilon,\delta}(T) = \Omega\left(\max\left\{\sqrt{s^* T \log(d/s^*)}, \frac{\log(d/s^*) \sqrt{s^*\log(1/\delta)}}{\epsilon}\right\}\right).
\end{align*}

\paragraph{Choice of $\rho$:} Consider $\epsilon - \sqrt{\rho_0\log(1/\delta)} -\rho_0/4 + 5 \rho_0^{1/4}=0$, recall that we need the positive root of this equation. Putting $t=\rho_0^{1/4}$, we get the quartic equation
$$f(t):=\frac{t^4}{4} + \sqrt{\log(1/\delta)} t^2 - 5 t - \epsilon = 0.$$
Note that $f(0)=-\epsilon<0$. Let $t_0=\tilde{c}\sqrt{\epsilon} / (\log(1/\delta))^{1/4}$. Note that if $\tilde{c}\geq 20^{1/3}$, then $f(t_0)>0$. To see this,
$$f(t_0) = \left[\frac{\tilde{c}^4\epsilon^2}{4\log(1/\delta)} - \frac{5\tilde{c}\sqrt{\epsilon}}{(\log(1/\delta))^{1/4}}\right] + [\tilde{c}^2\epsilon - \epsilon] > 0, \text{ since } \tilde{c}^3 > 20 \geq 20\left(\frac{\epsilon^2}{\log(1/\delta)}\right)^{-3/4}.$$
This ensures that there exists a positive root $\hat{t}$ of $f$ such that $\hat{t}\leq t_0$. This also indicates that there exists a positive root $\rho_0$ of  $\epsilon - \sqrt{\rho_0\log(1/\delta)} -\rho_0/4 + 5 \rho_0^{1/4}=0$ such that $\rho_0^{1/4} \leq 20^{1/3}\sqrt{\epsilon} / (\log(1/\delta))^{1/4}$, i.e., $\rho_0 \leq 20^{4/3}\epsilon^2 / \log(1/\delta)$.

\section{\uppercase{Proof of Results in Section}~\ref{section: general result}}

\subsection{Results on Peeling}
We first discuss some properties of the peeling mechanism. Given a vector $v\in\bbR^d$, we denoted $\cP_s(v)$ to be the peeled vector (i.e., the output of Algorithm \ref{alg: peeling} for the chosen hyperparameters). To this end, we denote $\tilde{P}_s(v,\bW)$ to be the vector $\tilde{P}_s(v)$ in the algorithm where $\bW$ represents a realization of $(\bw_1,\dots,\bw_s)$, the collection of Laplace noise vectors used in the iterative phase of the algorithm. Similarly, we denote $\cP_s(v;\bW,\tilde{\bw})$ to denote the output of peeling $v$ for a particular realization of all the Laplace random variables.

We use the following lemma (see Lemma A.3 of \cite{cai2021cost})
\begin{lemma}[Basic Peeling]\label{lemma: peeling}
    For any index set $I\subset[d]$ and $v\in \bbR^d$ with $\supp(v) \subseteq I$ and any $\hat{v}\in\bbR^d$ such that $\norm{\hat{v}}_0\leq \hat{s}\leq s$, we have that for every $c>0$
    \begin{equation}
        \norm{\tilde{P}_s(v, \bW) - v}_2^2 \leq \left(1+\frac{1}{c}\right)\frac{|I|-s}{|I|-\hat{s}}\norm{\hat{v} - v}_2^2 + 4(1+c)\sum_{i\in[s]} \norm{\bw_i}_{\infty}^2.
    \end{equation}
\end{lemma}

We note that if $\supp(\tilde{P}_s(v,\bW))=S$, then $\tilde{P}_s(v,\bW)=\tilde{P}_s(v_S,\bW)$. The above lemma allows us to show that peeling is an \textit{almost-contraction}.

\begin{lemma}\label{lemma: peeling contraction}
Given $\theta^*\in\bbR^d$ with $\norm{\theta^*}_0\leq s^*$ and $\theta\in \bbR^d$ with $\supp(\theta^*)\subseteq\supp(\theta)$ with $\norm{\theta}_0\leq s+s^*$, for any $c_1,c_2>0$ and $c>0$ as in Lemma \ref{lemma: peeling}, we have
        $$\norm{\cP_s(\theta, \bW, \tilde{\bw}) - \theta^*}_2^2 \leq a\norm{\theta -\theta^*}_2^2 + b,$$
        where
        \begin{align*}
            a &= \left(1 + \frac{1}{c_1}\right)\left[\left(1 + \frac{1}{c_2}\right)\left(1 + \frac{1}{c}\right) \frac{s^*}{s} + (1+c_2)\right]\\
            b &= (1+c_1)\norm{\tilde{\bw}_S}_2^2 + 4\left(1 + \frac{1}{c_1}\right)\left(1 + \frac{1}{c_2}\right)(1+c)\sum_{i\in [s]} \norm{\bw_i}_{\infty}^2
        \end{align*}
\end{lemma}
\begin{proof}
    Note that $\cP_s(\theta, \bW, \tilde{\bw}) = \tilde{P}_s(\theta, \bW) + \tilde{\bw}_S$, where $S = \supp(\cP_s(\theta, \bW, \tilde{\bw})$. Two applications of Young's inequality gives
    \begin{align*}
        \norm{\cP_s(\theta, \bW, \tilde{w}) - \theta^*}_2^2 &\leq \norm{\tilde{P}_s(\theta, \bW) - \theta^* + \tilde{\bw}_S}_2^2 \\
        &\leq \left(1 + \frac{1}{c_1}\right)\norm{\tilde{P}_s(\theta, \bW) - \theta^*}_2^2 + (1+c_1)\norm{\tilde{\bw}_S}_2^2 \\
        &= \left(1 + \frac{1}{c_1}\right)\norm{\tilde{P}_s(\theta, \bW) - \theta + \theta - \theta^*}_2^2 + (1+c_1)\norm{\tilde{\bw}_S}_2^2 \\
        &\leq \left(1 + \frac{1}{c_1}\right)\left(1 + \frac{1}{c_2}\right)\norm{\tilde{P}_s(\theta,\bW) - \theta}_2^2 + \left(1 + \frac{1}{c_1}\right)(1+c_2) \norm{\theta - \theta^*}_2^2 + (1+c_1) \norm{\tilde{\bw}_S}_2^2
    \end{align*}
    Now, let $S' := \supp(\theta)\cup \supp(\theta^*) = \supp(\theta)$, then applying Lemma \ref{lemma: peeling} with $\hat{v}=\theta^*$ and $S'\subset I$ with $|I|= s+s^*$, we have
    $$\norm{\tilde{P}_s(\theta,\bW)-\theta}_2^2 \leq \left(1 + \frac{1}{c}\right)\frac{s^*}{s}\norm{\theta-\theta^*}_2^2 + 4(1+c)\sum_{i\in[s]}\norm{\bw_i}_{\infty}^2.$$
    Plugging this bound above, we get the result.
\end{proof}

Th above result is valid for all choices of $c_1, c_2, c>0$.

\subsection{Proof of Theorem \ref{thm: general result}} \label{app: proof of general result}

\textbf{Step 1: Defining good events.} Assume that $\norm{x_i}_{\infty}\leq \xmax$ and $\theta^*\in\bbR^d$ is any arbitrary point such that  $\norm{\theta^*}_1\leq \bmax$. Now, we define the following \textit{good} events:
\allowdisplaybreaks
\begin{align*}
    \cE_1 &\triangleq \left\{{\inf_{\theta} \phi_{\min}(Cs^*, \nabla^2 \cL(\theta)) \geq \underline{\kappa}, \sup_{\theta} \phi_{\max}(Cs^*, \nabla^2 \cL(\theta)) \leq \bar{\kappa}}\right\} \\
    \cE_2 &\triangleq \left\{\norm{\nabla \cL(\theta^*)}_{\infty} \leq \bar{g}\right\} \\
    \cE_3 &\triangleq \left\{\max_{m\in[M]} \bW_m \leq \bar{W}_{(\epsilon,\delta)}\right\} \\
    \cE_4 &\triangleq \{\clip_R(y_i) = y_i~~\forall i\in[n] \mid R=\xmax\bmax+\sigma\sqrt{2\log n}\}.  
\end{align*}
Here, $\sigma$ is the sub-Gaussian parameter for the errors $\xi$.
$\bW_m=\sum_{i\in[s]} \norm{\bw_i^{(m)}}_{\infty}^2 + \norm{\tilde{\bw}_{S^m}^{(m)}}_2^2$ is the total amount of noise injected in the $m$th iteration (through the peeling $\cP_s$ step) and both the gradient $\nabla \cL$ and Hessian $\nabla^2 \cL$, appearing in $\cE_3$ and $\cE_1$ respectively, are with respect to the data $(X,\clip_R(y))$. $\phi(k, A)$ represents the $k$-sparse eigenvalue of matrix $A$. We denote $\kappa\triangleq \bar{\kappa}/\underline{\kappa}$ as the condition number of the Hessian of the loss $\cL$.

\textbf{Step 2: Implication of good events.} We observe that we can prove the upper bound if the good events are true. 
By union bound, the good events hold simultaneously with probability at least 
$$\Pr(\cE_1\cap\cE_2\cap\cE_3\cap\cE_4)\ge 1 - \sum_{i=1}^4 \Pr(\cE_i^c),$$
which would prove the theorem.

\textbf{Step 3: Bounding per-step error $\|\theta^{m+1} - \theta\|_2$ at step ${m+1}$ with respect to the per-step error in step $m$ under good events.} 

(a) \textit{Contraction by projection.} First, we observe that
\begin{align}
    \norm{\theta^{m+1} - \theta^*}_2^2 &= \norm{\Pi_C(\tilde{\theta}^{m+1}) - \Pi_C(\theta^*)}_2^2 {\leq} \norm{\tilde{\theta}^{m+1} - \theta^*}_2^2. \label{ineq: projection}
\end{align}
The last inequality is true since $\Pi_C$ is a projection on a convex set, which makes it a contraction~\citep[Theorem 1.2.1]{schneider2014convex}, i.e.
\begin{equation*}
    \norm{\Pi_C(u) - \Pi_C(v)}_2 \leq \norm{u-v}_2 \quad \forall u,v \in\bbR^d.
\end{equation*}

(b) \textit{Impact of peeling.} Now, using the notations in Algorithm \ref{alg:noisy-IHT}, let $S_m=\supp (\theta^m)$, $\tilde{S}_{m+1}=\supp(\tilde{\theta}^{m+1})$ and $S^*=\supp(\theta^*)$. Note that because of the projection $\Pi_C$, $S_{m+1}\subseteq \tilde{S}_{m+1}$.  Denote $I_m = S^* \cup \tilde{S}_{m+1}$. 

Now, we get
\begin{align*}
\norm{\tilde{\theta}^{m+1} - \theta^*}_2^2     &\overset{\text{defn}}{=} \norm{\cP_s(\theta^{m+0.5}, \bW^{(m)}, \tilde{\bw}^{(m)}) - \theta^*}_2^2 \\
    &= \norm{\cP_s(\theta^{m+0.5}_{I_m}, \bW^{(m)}, \tilde{\bw}^{(m)}) - \theta^*}_2^2 \\
    &{\leq} a\norm{\theta^{m+0.5}_{I_m} - \theta^*}_2^2 + b.
\end{align*}
The last inequality is a direct application of Lemma~\ref{lemma: peeling contraction}. Since $|I_m|\leq s+s^*$ and $\supp(\theta^{m+0.5}_{I_m})\subseteq \supp(\theta^*)$, we can apply Lemma~\ref{lemma: peeling contraction} with noises $\bW^{(m)}, \tilde{\bw}^{(m)}$ as the noise used in the $m$-th iteration and $a, b$ defined in the lemma.

(c) \textit{Impact of optimization step.} Now, we focus on bounding the impact of the optimization step in the term $\norm{\theta^{m+0.5}_{I_m} - \theta^*}_2^2$.
\begin{align*}
  \norm{\theta^{m+0.5}_{I_m} - \theta^*}_2^2
    &\overset{\text{defn}}{=}  \norm{\theta^m_{I_m} - \eta (\nabla \cL(\theta^m))_{I_m} - \theta^*}_2^2 \\
    &= \norm{\theta^m_{I_m} -\theta^* - \eta ((\nabla \cL(\theta^*))_{I_m} + \cH_{*I_m}(\theta^m_{I_m} - \theta^*))}_2^2 \\
    &= \norm{(I - \eta \cH_{*I_m})(\theta^m_{I_m} - \theta^*) - \eta (\nabla \cL(\theta^*))_{I_m}}_2^2.
\end{align*}
The first inequality is a direct consequence of the Algorithm~\ref{alg:noisy-IHT}.
In the second equality above, we denote $\cH = \int_0^1 \nabla^2 \cL (\theta^* + \gamma(\theta^m - \theta^*))d\gamma$ as the integral form of the remainder for the Taylor's expansion applied to $\nabla \cL(\theta^m)$. However, owing to the choice of $I_m$, we only need to restrict  columns of $\cH$ to $I_m$, which we denote by $\cH_{*I_m}$.

(d) \textit{Contraction due to convexity.}
Now, by applying H\"older's inequality, we observe
\begin{align*}
    &\norm{(I-\eta\cH_{*I_m})(\theta^m_{I_m}-\theta^*) - \eta\nabla\cL(\theta^*)_{I_m}}_2^2 \\
    &\leq \norm{I-\eta\cH_{*I_m}}_2^2\norm{\theta^m_{I_m}-\theta^*}_2^2 + \eta^2\norm{\nabla\cL(\theta^*)_{I_m}}_2^2 - 2\eta\innerprod{(I-\eta \cH_{*I_m})(\theta^m_{I_m}-\theta^*), \nabla\cL(\theta^*)_{I_m}} \\
    &\leq (1-\eta\underline{\kappa})^2 \norm{\theta^m_{I_m}-\theta^*}_2^2 + \eta^2\norm{\nabla\cL(\theta^*)_{I_m}}_2^2 + \eta\left(\frac{1}{c_3}\norm{\theta^m_{I_m}-\theta^*}_2^2 + c_3 \norm{I-\eta\cH_{*I_m}}_2^2\norm{\nabla\cL(\theta^*)_{I_m}}_2^2\right)\\
    &\overset{\eta\underline{\kappa}<1}{\leq} \left(1 - \eta\underline{\kappa}+\frac{\eta}{c_3}\right)\norm{\theta^m_{I_m}-\theta^*}_2^2 + \left(\eta^2 + \eta c_3 (1-\eta\underline{\kappa})^2\right)\norm{\nabla \cL(\theta^*)_{I_m}}_2^2 \tag{**} \\
    &\overset{c_3=7/6\underline{\kappa}}{\leq} \left(1 - \frac{\eta\underline{\kappa}}{7}\right)\norm{\theta^m_{I_m}-\theta^*}_2^2 + \left(\eta^2 + \frac{7\eta}{6\underline{\kappa}}(1 - \eta\underline{\kappa})\right) \norm{\nabla \cL(\theta^*)_{I_m}}_2^2 \\
    &= \left(1 - \frac{\eta\underline{\kappa}}{7}\right)\norm{\theta^m_{I_m}-\theta^*}_2^2 + \frac{\eta}{6\underline{\kappa}}\left(7 - \eta\underline{\kappa}\right) \norm{\nabla \cL(\theta^*)_{I_m}}_2^2
\end{align*}

Under $\cE_1$, we have the restricted strong convexity and restricted strong smoothness for $\cL$, i.e., for any $\theta\in\bbR^d$ such that $\norm{\theta}_0\lesssim s^*$,
\begin{equation}\label{ineq: RSS_RSC}
    \innerprod{\nabla \cL(\theta^*), \theta - \theta^*} + \underline{\kappa} \norm{\theta - \theta^*}_2^2 \leq \innerprod{\nabla \cL(\theta), \theta - \theta^*} \leq \innerprod{\nabla \cL(\theta^*), \theta - \theta^*} + \bar{\kappa} \norm{\theta - \theta^*}_2^2.
\end{equation}
We note that by algorithm, $\theta^m$ is at most $s$ sparse and we would choose $s \gtrsim \kappa^2 s^*$ (made precise later). Inequality $(**)$ above now follows by using the RSC parameter $\underline{\kappa}$ from \eqref{ineq: RSS_RSC} and arguing as before.

In the above, we require that $I-\eta\cH$ is positive definite, for which we need $\eta\bar{\kappa}<1$. 

(e) \textit{Retrieving $\theta^m$ from $\theta^m_{I_m}$.}

\begin{align*}
&\left(1 - \frac{5\eta\underline{\kappa}}{7}\right)\norm{\theta^m_{I_m} - \theta^*}_2^2 + \frac{\eta}{2\underline{\kappa}}(7 - 5\eta\underline{\kappa})\norm{\nabla \cL(\theta^*)_{I_m}}_2^2 \\
\leq &\left(1 - \frac{5\eta\underline{\kappa}}{7}\right)\norm{\theta^m - \theta^*}_2^2 + \frac{\eta}{2\underline{\kappa}}(7 - 5\eta\underline{\kappa})(s+s^*)\norm{\nabla \cL(\theta^*)}_{\infty}^2.
\end{align*}

The last line uses two facts: (i) $\norm{v}_2^2\leq \norm{v}_0\norm{v}_{\infty}^2$ and (ii) since $\theta^*$ is 0 outside $I_m$, we have $$\norm{\theta^m - \theta^*}_2^2 = \norm{\theta^m_{I_m} - \theta^*_{I_m}}_2^2 + \norm{\theta^m_{I_m^c} - \theta^*_{I_m^c}}_2^2 = \norm{\theta^m_{I_m} - \theta^*}_2^2 + \norm{\theta^m_{I_m^c}}_2^2 \geq \norm{\theta^m_{I_m} - \theta^*}^2.$$

(f) \textit{Combining results with tuned constants.} For such choices, combining all the above (and using the forms for $a$ and $b$ as obtained from Lemma \ref{lemma: peeling contraction}, we get
\begin{align}\label{eq: contraction1}
    \norm{\theta^{m+1}-\theta^*}_2^2 \leq A \norm{\theta^m-\theta^*}_2^2 + B
\end{align}
where
\begin{align}
    A &= \left(1 + \frac{1}{c_1}\right)\left[\left(1 + \frac{1}{c_2}\right)\left(1 + \frac{1}{c}\right) \frac{s^*}{s} + (1+c_2)\right]\left(1 - \frac{\eta\underline{\kappa}}{7}\right) \\
    B &= \underbrace{(1+c_1)}_{b_1}\norm{\tilde{\bw}_S^{(t)}}_2^2 + \underbrace{4\left(1 + \frac{1}{c_1}\right)\left(1 + \frac{1}{c_2}\right)(1+c)}_{b_2}\sum_{i\in [s]} \norm{\bw_i^{(t)}}_{\infty}^2 \nonumber\\
    &\quad\quad + \underbrace{\left(1 + \frac{1}{c_1}\right)\left[\left(1 + \frac{1}{c_2}\right)\left(1 + \frac{1}{c}\right) \frac{s^*}{s} + (1+c_2)\right] \frac{\eta}{6\underline{\kappa}}(7 - \eta\underline{\kappa})(s+s^*)}_{b_3}\norm{\nabla \cL(\theta^*)}_{\infty}^2
\end{align}

Now, with the choices $\eta=1/2\bar{\kappa}, c_1=56\kappa-3, c_2=1/8(14\kappa-1), c=1$ and $s \ge 16(14\kappa-1)(112\kappa-7)s^*$, we have for suitably large constants $C_1, C_2>0$
\begin{align*}
    1 - \frac{\eta\underline{\kappa}}{7} &= 1 - \frac{1}{14\kappa} \\
    A & \le 1 - \frac{1}{28\kappa} := \rho \\
    b_1 &= 56\kappa-2 \leq C_1\kappa \\
    b_2 &= 8\frac{(56\kappa-2)(112\kappa-7))}{(56\kappa-3)} \leq C_1\kappa \\
    b_3 &\leq C_2 \frac{\kappa}{\underline{\kappa}^2}(s +s^*).
\end{align*}
Plugging these into the upper bound in Equation \eqref{eq: contraction1}, we have
\begin{align*}
    \norm{\theta^{m+1}-\theta^*}_2^2 \leq \rho \norm{\theta^m-\theta^*}_2^2 + C_1\kappa \underbrace{\left(\norm{\tilde{\bw}_{S^m}^{(m)}}_2^2 + \sum_{i\in[s]} \norm{\bw_i^{(m)}}_{\infty}^2\right)}_{\bW_m} + C_2\frac{\kappa}{\underline{\kappa}^2}(s +s^*)\norm{\nabla \cL(\theta^*)}_{\infty}^2.
\end{align*}

\textbf{Step 4: Convergence rate over multiple iterations.} Iterating the above inequality for $m=0,\dots,M-1$, we obtain
\begin{align}
    \norm{\theta^{M}-\theta^*}_2^2 &\leq \rho^M \norm{\theta^0-\theta^*}_2^2 + C_1\kappa \sum_{m=0}^{M-1} \rho^m \bW_m + C_2\frac{\kappa}{\underline{\kappa}^2}(s+s^*)\norm{\nabla \cL(\theta^*)}_{\infty}^2 \sum_{m=0}^{M-1} \rho^m \nonumber\\
    &\leq 2\rho^M \bmax^2 + C_1\kappa \left(\max_m \bW_m\right) \frac{1-\rho^M}{1-\rho} + C_2\frac{\kappa}{\underline{\kappa}^2} (s+s^*)\norm{\nabla \cL(\theta^*)}_{\infty}^2 \frac{1-\rho^M}{1-\rho}
\end{align}
where we assume $\norm{\theta^0}_1\leq \bmax$. Finally using the good sets $\cE_2$ and $\cE_3$ and using the choice $M=28\kappa\log (2\bmax^2 n)$, we get
\begin{equation}
    \norm{\theta^M-\theta^*}_2^2 \leq \frac{1}{n} + C\kappa^2\bar{W} + C'\frac{\kappa^2}{\underline{\kappa}^2}(s+s^*)\bar{g}
\end{equation}
where the constants $C=28C_1, C'=28C_2$.

\section{\uppercase{Proofs in Section}~\ref{sec: upper bound}}

\subsection{Proof of Theorem \ref{thm: fliphat DP}}
Recall that $\br_t = (r_1(t), r_2(t), \ldots, r_K(t))^\top$ and $\cC_t := (x_1(t), \ldots, x_{K}(t))^\top$. Consider two $\tau$-neighboring datasets $\cD:= \{(\br_t, \cC_t)\}_{t \in [T]}$ and $\cD^\prime = \{(\br_t^\prime, \cC_t^\prime)\}_{t \in [T]}$. We denote by $\hat{a}_t$ the random map induced by the FLIPHAT at time point $t$ that takes a set of observed reward-context pairs from the previous episode as input and spits out an action $a \in [K]$ at the $t$th time point. Define the function $\gamma: [T] \mapsto \bbN\cup\{0\}$ as follows:
\[
\gamma(t) := \sum_{\ell=0}^{\floor{\log_2 T}}  \ell \ind\{t_\ell \le t < t_{\ell+1}\}, \quad \text{where $t_\ell =  2^\ell$.}
\]
Essentially, $\gamma(t)$ identifies the index of the episode in which the time point $t$ lies. Now recall that at the $t$th time point, FLIPHAT takes the action $\hat{a}_t(\cC_t; \cH_{\gamma(t)-1}^{\hat{\ba}}) \in [K]$, where $\cH_{\gamma(t)-1}^{\hat{\ba}} = \{(r_{\hat{a}_u}(u), x_{\hat{a}_u}(u)); \gamma(t)-1 \le u < \gamma(t)\}$ are the observed rewards and contexts until time $t$ under the (random) action sequence $\{\hat{a}_u\}_{u \le t-1}$. Also, we define $\cH_{-1}^{\hat{\ba}} = \emptyset$. In FLIPHAT,  $\hat{a}_1$ is chosen uniformly randomly over $[K]$.
For conciseness, we will use $\cA$ to denote the FLIPHAT protocol.
Now we fix a sequence of action $\ba_T = (a_1, \ldots, a_T)$, and define $\ba_{-\tau}:= (a_1, \ldots, a_{\tau-1}, a_{\tau+1}, \ldots, a_T)$. WLOG, let us first assume that $\gamma(\tau) = \ell_0$. Using the sequential nature of FLIPHAT, we get 
\begin{align*}
\pr(\cA_{-\tau}(\cD) = \ba_{-\tau}) = \prod_{\ell\ge 0} \pr(\hat{a}_t( \cC_t; \cH_{\ell-1}^\ba) = a_t; t_\ell \le t < t_{\ell+1}, t \ne \tau) .
\end{align*}
Now, we will break the product into three parts based on the episode index $\ell$: $\ell \le \ell_0, \ell = \ell_0 +1$, and $\ell > \ell_0 + 1$. 
Also, recall that changing the reward and context at time $\tau$ only affects the private estimate $\hat{\beta}_{\ell_0+1}$ as it only depends on the data $\cH^\ba_{\ell_0}:=  \{(r_{a_t}(t), x_{a_t}(t))\}_{t = t_{\ell_0}}^{t_{\ell_0+1}-1}$.

\paragraph{Case $\ell \le \ell_0$:}
Note that for the fixed action sequence $\ba_{-\tau}$, $(\br_t, \cC_t) = (\br_t^\prime, \cC_t^\prime)$ for all $t \in [t_{\ell_0 + 1}]\setminus \{\tau\}$, and we also have $\cH^\ba_{\ell-1} = \cH^{\prime \ba}_{\ell-1}$. 
First, fix a time point $t$ in the $\ell$th episode.
By the design of $\cA$, the randomness in the action $\hat{a}_t$ only comes through the N-IHT estimator $\hat{\beta}_\ell$. For clarity,  let us use the notation $\hat{\beta}_{\ell} (\cH^\ba_{\ell-1})$ to emphasize the dependence of $\cH^\ba_{\ell-1}$ in the N-IHT estimator. 
Let $F_{\hat{\beta}_\ell}(\cdot)$ be the distribution function of $\hat{\beta}_\ell(\cH^\ba_{\ell-1})$, and $F^\prime_{\hat{\beta}_\ell}(\cdot)$ be the distribution function of $\hat{\beta}_\ell(\cH^{\prime\ba}_{\ell-1})$. Since $\cH^\ba_{\ell-1} = \cH^{\prime \ba}_{\ell-1}$ under two neighboring datasets, the distribution functions $F_{\hat{\beta}_\ell}$ and $F^\prime_{\hat{\beta}_\ell}$ are identical. Moreover, $\hat{a}_t( \cC_t; \cH_{\ell-1}^\ba) = \argmax_{a \in [K]} x_a(t)^\top \hat{\beta}_\ell(\cH^\ba_{\ell-1}) =: \sfA(\hat{\beta}_\ell(\cH^\ba_{\ell-1}); \cC_t)$. Also define the set $O_t:= \{z \in \bbR^d : \sfA(z; \cC_t) = a_t\}$. Hence, we have the following for all $\ell\le \ell_0$:

\begin{align*}
    & \pr(\hat{a}_t( \cC_t; \cH_{\ell-1}^\ba) = a_t; t_\ell \le t < t_{\ell+1}, t \ne \tau) \\
    &=\int_{\bbR^d}\ind\{\sfA(z; \cC_t) = a_t; t_\ell \le t < t_{\ell+1}, t \ne \tau\} \;dF_{\hat{\beta}_\ell}(z)\\
    & = \pr \left(\hat{\beta}_\ell(\cH^\ba_{\ell-1}) \in \cap_{\substack{t_\ell \le t < t_{\ell+1}\\ t \ne \tau}} O_t\right)\\
    & = \pr \left(\hat{\beta}_\ell(\cH^{\prime\ba}_{\ell-1}) \in \cap_{\substack{t_\ell \le t < t_{\ell+1}\\ t \ne \tau}} O_t\right)\\
    & = \pr(\hat{a}_t( \cC_t; \cH_{\ell-1}^{\prime\ba}) = a_t; t_\ell \le t < t_{\ell+1}, t \ne \tau)
\end{align*}

\paragraph{Case $\ell> \ell_0+1$:}
The argument in this is exactly the same as the previous case. Note that $(\br_t, \cC_t) = (\br_t^\prime, \cC_t^\prime)$ for all $t \ge t_{\ell_0 + 2}$, and we also have $\cH^\ba_{\ell-1} = \cH^{\prime \ba}_{\ell-1}$. Recall that due to the forgetful nature of FLIPHAT, $\hat{\beta}_\ell$ does not depend on $\{\cH^\ba_{j}\}_{j \le \ell_0}$ at all. Therefore, following the same reasoning, we have 
\[
\pr(\hat{a}_t( \cC_t; \cH_{\ell-1}^{\ba}) = a_t; t_\ell \le t < t_{\ell+1}, t \ne \tau) = \pr(\hat{a}_t( \cC_t; \cH_{\ell-1}^{\prime\ba}) = a_t; t_\ell \le t < t_{\ell+1}, t \ne \tau)
\]
for all $\ell > \ell_0+1$.

\paragraph{Case $\ell = \ell_0 + 1$:} First note that $\cH^\ba_{\ell_0}$ and $\cH^{\prime \ba}_{\ell_0}$ are neighboring datasets as they differ only in one entry. Also, recall that by the design of the algorithm, $\hat{\beta}_{\ell_0 + 1}$ is an $(\epsilon, \delta)$-DP estimator. Hence, we have

\begin{align*}
       & \pr(\hat{a}_t( \cC_t; \cH_{\ell_0}^\ba) = a_t; t_{\ell_0+1} \le t < t_{\ell_0+2}, t \ne \tau) \\
    &=\int_{\bbR^d}\ind\{\sfA(z; \cC_t) = a_t; t_{\ell_0+1} \le t < t_{\ell_0+2}, t \ne \tau\} \;dF_{\hat{\beta}_\ell}(z)\\
    & = \pr \left(\hat{\beta}_{\ell_0+1}(\cH^\ba_{\ell_0}) \in \cap_{\substack{t_{\ell_0+1} \le t < t_{\ell_0+2}\\ t \ne \tau}} O_t\right)\\
    & \le e^\epsilon\pr \left(\hat{\beta}_{\ell_0+1}(\cH^{\prime\ba}_{\ell_0}) \in \cap_{\substack{t_{\ell_0+1} \le t < t_{\ell_0+2}\\ t \ne \tau}} O_t\right) + \delta\\
    & = e^\epsilon\pr(\hat{a}_t( \cC_t; \cH_{\ell-1}^{\prime\ba}) = a_t; t_\ell \le t < t_{\ell+1}, t \ne \tau) + \delta.
\end{align*}

Combining all these cases we have:
\begin{align*}
    &\pr(\cA_{-\tau}(\cD) = \ba_{-\tau})\\
    &= \prod_{\ell\ge 0} \pr(\hat{a}_t( \cC_t; \cH_{\ell-1}^\ba) = a_t; t_\ell \le t < t_{\ell+1}, t \ne \tau)\\
    & \le e^\epsilon \prod_{\ell\ge 0} \pr(\hat{a}_t( \cC_t; \cH_{\ell-1}^{\prime\ba}) = a_t; t_\ell \le t < t_{\ell+1}, t \ne \tau) + \delta \prod_{\ell\ne \ell_0 + 1} \pr(\hat{a}_t( \cC_t; \cH_{\ell-1}^\ba) = a_t; t_\ell \le t < t_{\ell+1}, t \ne \tau)\\
    &\le e^\epsilon \pr(\cA_{-\tau}(\cD^\prime) = \ba_{-\tau}) + \delta.
\end{align*}

\subsection{Proof of Proposition \ref{prop: estimation episode}}\label{app: bandit corollary}
We use Theorem \ref{thm: general result}, with $\bar{g}=2\sigma\xmax\sqrt{\log d/N_\ell}$ and $\bar{W}=K' R^2 (s^* \log d \log N_\ell)^2 \log ((\log N_\ell)/\delta) / (N_\ell^2 \epsilon^2)$ for a suitably large $K>0$. Plugging these into the upper bound, we have the desired result given the mentioned choices. We are only left to control the probabilities of the events
\begin{align*}
    \cE_{2\ell} &= \left\{\norm{\nabla \cL(\beta^*|X_\ell, \clip_R(y_\ell)}_{\infty} \leq 2\sigma\xmax\sqrt{\frac{\log d}{N_\ell}}\right\} \\
    \cE_{3\ell} &= \left\{\max_m \bW_m \leq K'.\frac{(s^*)^2\log^2 d\log \left(\frac{\log N_\ell}{\delta}\right)\log^3 N_\ell}{N_\ell^2\epsilon^2}\right\} \\
    \cE_{4\ell} &= \left\{\clip_R(y_\ell)=y_\ell\right\}.
\end{align*}
Using the argument as in Section \ref{section: general result}, due to $\cE_{4\ell}$, we can use $\nabla \cL(\beta^*)$ evaluated on the data $(X_\ell,y_\ell)$ instead of $(X_\ell, \clip(y_\ell))$. To see this, consider $\tilde{\cE}=\{\norm{\nabla \cL(\beta^*|X_\ell,y_\ell)}_{\infty}\leq 2\sigma\xmax\sqrt{(\log d)/N_\ell}\}$. Then $\Pr(\cE_{2\ell}^c)=\Pr(\cE_{2\ell}^c\cap \cE_{4\ell}) + \Pr(\cE_{2\ell}^c\cap \cE_{4\ell}^c)$. For the first term $\cE_{2\ell}\cap\cE_{4\ell}\subset \tilde{\cE}$. For the second term, we can simply upper bound by $\Pr(\cE_{4\ell}^c)$, thus, we obtain $\Pr(\cE_{2\ell}^c) \leq \Pr(\tilde{\cE}^c) + \Pr(\cE_{4\ell}^c)$. Thus, it suffices to have upper bounds for $\Pr(\tilde{\cE}^c)$ and $\Pr(\cE_{4\ell}^c)$, to guarantee that $\cE_{2\ell}$ occurs with high-probability.

\paragraph{Understanding $\cE_{4\ell}$} Recall that $\cE_{4\ell}=\{\clip_R(r_t)=r_t,  \, \forall t\in[t_\ell,t_{\ell+1})\}=\{|r_t|\leq R, \, \forall t\in [t_{\ell}, t_{\ell+1})\}$. Now, we have
    $$|r_t| = |x_{a_t}(t)^\top\beta^*+\varepsilon_t| \leq \norm{x}_{\infty} \norm{\beta^*}_1 + |\varepsilon(t)| \leq \xmax \bmax + |\varepsilon(t)|.$$
    Recall $R = \xmax\bmax + \sigma\sqrt{2\log N_\ell}$, hence we have
    $$|r_t|\leq R \iff |\varepsilon(t)| \leq \sigma \sqrt{2\log N_\ell}.$$
    The $\varepsilon'$s are i.i.d. Sub-Gaussian$(\sigma)$. Thus, we have
    $$\Pr\left(\max_{t\in[t_\ell,t_{\ell+1}]} |\varepsilon(t)| > \sigma \sqrt{2\log N_\ell}\right) \leq N_\ell\exp(-2\sigma^2 \log N_\ell / (2\sigma^2)) = 1/N_\ell,$$
    which implies that 
    $$\Pr(\cE_{4\ell}^c) \leq 1/N_\ell.$$

\paragraph{Understanding $\cE$:} Recall that $\cE = \left\{\norm{\nabla\cL (\beta^*)}_\infty \leq 2 \sigma\xmax \sqrt{\frac{\log d}{N_\ell}}\right\}$, where the gradient is computed for $(X_\ell,y_\ell)$. Note that $\nabla\cL(\beta^*) = X_\ell^\top (y_\ell - X_\ell\beta^*)/N_\ell = X_\ell^\top \boldsymbol{\epsilon}_\ell/N_\ell$, where  $\boldsymbol{\epsilon}_\ell$ is the $N_\ell-$length vector of noise $\epsilon_t$ for $t\in[t_\ell, t_{\ell+1})$. Thus, $\norm{\nabla \cL(\beta^*)}_\infty=\max_{j\in[d]} |\boldsymbol{\epsilon}_\ell^\top X^{(j)}_\ell|/N_\ell$. Now, under the joint distribution of the contexts and noise under the bandit environment, we have \citep[Proposition 3]{chakraborty2023thompson}
\begin{equation}
    \Pr\left(\max_{j\in[d]} \frac{|\boldsymbol{\epsilon}_\ell^\top X_\ell^{(j)}|}{N_\ell} \leq \sigma \xmax \sqrt{\frac{\gamma^2+2\log d}{N_\ell}}\right) \geq 1 - 2\exp(-\gamma^2 / 2).
\end{equation}
Plugging $\gamma^2=2\log d$, we get
\begin{equation}
    \Pr\left(\cE^c\right) \leq 2\exp (-2\log d).
\end{equation}

\paragraph{Understanding $\cE_{3\ell}$:}  The following tail-bound for Laplace distribution is taken from \citep[Lemma A.1]{cai2021cost}
\begin{lemma}
    Consider $\boldsymbol{w}\in\bbR^k$ with $w_1,w_2,\dots,w_k\overset{iid}{\sim} \text{Laplace}(\lambda)$. For every $C>1$, we have
    \begin{align}
        P(\norm{\boldsymbol{w}}_2^2 > kC^2\lambda^2) &\leq ke^{-C} \label{eq: Laplace tail 2 norm}\\
        P(\norm{\boldsymbol{w}}_\infty^2 > C^2\lambda^2\log^2 k) &\leq e^{-(C-1)\log k} \label{eq: Laplace tail inf norm}
    \end{align}
\end{lemma}
Recall in our case, we have $\bW_m = \left(\sum_{i\in[s]} \norm{\bw_i^{(m)}}_\infty^2 + \norm{\tilde{\bw}^{(m)}_{S}}_2^2\right)$, where each $\bw_i^{(m)}\in \bbR^d$ and $\tilde{\bw}_{S}$ has $|S|=s$ non-zero coordinates. For the moment, let us drop the sub/super-scripts involving $m$. 

For control on $\norm{\tilde{\bw}_S}_2^2$, we note that in our case $\lambda=\eta B T\sqrt{3s\log(M_\ell/\delta)}/(N\epsilon)$, with $B \asymp R_\ell$ and $M_\ell=28\kappa\log (2\bmax^2 N_\ell) \asymp \log N_\ell$, which gives $$\lambda \asymp (\log N_\ell)^2 s^{1/2}\sqrt{\log ((\log N_\ell)/\delta)}/N_\ell\epsilon.$$ Using \eqref{eq: Laplace tail 2 norm} from the above lemma with $k=s $, $\lambda$ as above and $C=c_1'\log d$ so that 
$$ k C^2\lambda^2 = s \times 16\log^2 d \times \frac{ s\log((\log N_\ell)/\delta) \log^3 N_\ell}{N_\ell^2\epsilon^2} \asymp \frac{s^2\log^2 d\log((\log N_\ell)/\delta) \log^2 N_\ell}{N_\ell^2 \epsilon^2}$$
Thus, for suitably large constant $K^\prime>0$ and $c_1'=8$ we have
$$\Pr\left(\norm{\tilde{\bw}_S}_2^2 > K^{\prime}. \frac{s^2\log^2 d\log((\log N_\ell)/\delta) \log^3 N_\ell}{N_\ell^2 \epsilon^2}\right) \leq \frac{s}{d^{8}}.$$
For control on $\norm{\bw_i}_{\infty}$, we use \eqref{eq: Laplace tail inf norm} with $k=d$, $\lambda$ as before and $C=c_2'+1$
$$C^2\lambda^2\log^2 k = C^2\times \frac{s\log((\log N_\ell)/\delta) \log^3 N_\ell}{N_\ell^2\epsilon^2} \times \log^2 d \asymp \frac{s \log^2 d \log((\log N_\ell)/\delta) \log^3 N_\ell}{N_\ell^2 \epsilon^2}$$
and hence for sufficiently large $K'>0$ and $c_2'=9$
$$P\left(\norm{\bw_i}_\infty^2 > K'. \frac{s \log^2 d \log((\log N_\ell)/\delta) \log^3 N_\ell}{N_\ell^2 \epsilon^2}\right) \leq \frac{1}{d^{8}}.$$
Thus, we get for each $m\in[M_\ell]$
\begin{align*}
    P\left(\bW_m \leq K'. \frac{ s^2 \log^2 d \log((\log N_\ell)/\delta) \log^3 N_\ell}{N_\ell^2\epsilon^2}\right) \geq \left(1 - \frac{s}{d^{8}}\right)\left(1 - \frac{1}{d^{8}}\right)^s \geq 1 - \frac{2s}{d^{8}}.
\end{align*}
Now, by union bound,
\begin{align*}
    &P\left(\max_{m\in[M_\ell]} \bW_m > K'. \frac{s^2 \log^2 d \log((\log N_\ell)/\delta) \log^3 N_\ell}{N_\ell^2\epsilon^2}\right) \\
    &\leq \sum_m P\left( \bW_m > K'. \frac{s^2 \log^2 d \log((\log N_\ell)/\delta) \log^3 N_\ell}{N_\ell^2\epsilon^2}\right) \\
    &\leq \frac{M_\ell \times 2s}{d^{8}} = O(1/d^6)
\end{align*}
where the last line follows from $M_\ell\asymp N_\ell \lesssim d$ and $s  \lesssim d$.

Thus, we obtain $P(\cE_4^c) \leq c_3' \exp (- c_4' \log (d))$ for some constants $c_3', c_4'>0$.

\subsection{Proof of Regret Bound}
\paragraph{Regret Bound:}


First, we consider the good event 
$\cA_\ell := \cap_{i =1}^4\cE_{i,\ell}$, where $\cE_{i,\ell}$ are the corresponding events of $\cE_i$ consisting of the data from $\ell$th episode. Also, recall that we need $N_\ell \gtrsim s^* \log d$ in order to hold $\cE_{1,\ell}$ with high probability. This implies that we need $\ell \ge  \log (s^* \log d) + A =: L$ for a large enough positive constant $A$. Let $\ell(t):= \ell$ if $t_{\ell} \leq t < t_{\ell+1}$. Also, note that $\log(\log N_{\ell-1}/\delta) \le \log(1/\delta) \log N_{\ell-1}$ if $\delta < e^{-1}$.


\paragraph{Problem independent bound:}
Let us define $\kappa_0 = \max\{\kappa, \kappa/\underline{\kappa}\} \asymp  K^2 \log K$.
\begin{align*}
   &\bbE \{R(T)\} = \sum_{t=1}^T \bbE \left(x_{t,a^*_t}^\top \beta^* - x_{t,a_t}^\top\beta^*\right) \\
    &= \underbrace{2 x_{\max}\norm{\beta^*}_1\sum_{0\leq \ell < L} N_{\ell}}_{\text{first } L \text{ episodes}} +  \sum_{t=t_L}^T \bbE\left[\left(x_{t,a^*_t}^\top \beta^* - x_{t,a_t^*}^\top\hat{\beta}_{\ell(t)}\right) + \underbrace{\left(x_{t,a_t^*}^\top\hat{\beta}_{\ell(t)} - x_{t,a_t}^\top\hat{\beta}_{\ell(t)}\right)}_{\leq 0 \text{ by how } a_t \text{ is chosen}} + \left(x_{t,a_t}^\top\hat{\beta}_{\ell(t)} - x_{t,a_t}^\top\beta^*\right)\right] 
    \\
    &\leq 2 x_{\max}b_{\max}(2^L-1)+ \sum_{\ell\geq L} \sum_{t_\ell\leq t<t_{\ell+1}} \bbE\left[\left(x_{t,a^*_t}^\top \beta^* - x_{t,a_t^*}^\top\hat{\beta}_{\ell(t)}\right) \ind_{\cA_\ell}+\left(x_{t,a_t}^\top\hat{\beta}_{\ell(t)} - x_{t,a_t}^\top\beta^*\right) \ind_{\cA_\ell}\right] + \xmax \bmax \pr(\cA_\ell^c)\\
    &\leq 2 x_{\max}b_{\max} (2^L-1)+ \sum_{\ell\geq L} \sum_{t_\ell\leq t<t_{\ell+1}} 2x_{\max}\bbE \norm{\beta^* - \hat{\beta}_{\ell(t)}}_1 \ind_{\cA_\ell}  + \xmax \bmax \pr(\cA_\ell^c)\\
    & \lesssim 2 x_{\max}b_{\max} (2^L - 1) + 2x_{\max}\sum_{\ell \geq L} (t_{\ell+1}-t_\ell) \sigma^2 \kappa_0 \left( \sqrt{\frac{(s + s^*)^2 \log d}{N_{\ell-1}}} + \sqrt{\frac{(s + s^*)^3 (\log d)^2 \log(\log N_{\ell-1}/\delta) \log^3 N_{\ell-1}}{N_{\ell-1}^2 \epsilon^2}}\right)\\
    & \quad + O(1)\\
    &\leq x_{\max}b_{\max} (2^L - 1) + 4 \sigma \kappa_0 x_{\max} \sum_{\ell\geq L} \left[ (s + s^*) \sqrt{N_{\ell-1} \log d} +  \sqrt{\frac{(s + s^*)^3 (\log d)^2 \log(1/\delta) \log^4 N_{\ell-1}}{\epsilon^2}} \right]  \\
    &\leq x_{\max}b_{\max} (2^L - 1) + 4\sigma^2 \kappa_0 x_{\max} (s + s^*) \sqrt{\log d} \sum_{\ell\geq L} 2^{(\ell-1)/2} + 4\sigma \kappa_0 x_{\max}\sqrt{\frac{(s + s^*)^3 (\log d)^2 \log(1/\delta) }{\epsilon^2}} \sum_{\ell \ge L} (\ell-1)^{2}\\
    &\lesssim x_{\max}b_{\max} (2^L - 1) + 4\sigma \kappa_0 x_{\max} (s + s^*) \sqrt{T\log d} + 4\sigma \kappa_0  \xmax \sqrt{\frac{(s + s^*)^3 (\log d)^2 \log(1/\delta) \log^6 T }{\epsilon^2}}\\
    & \lesssim 4\sigma \kappa_0 x_{\max} (s + s^*) \sqrt{T\log d} + 4\sigma \kappa_0  \xmax \sqrt{\frac{(s + s^*)^3 (\log d)^2 \log(1/\delta) \log^6 T }{\epsilon^2}}.
\end{align*}


\subsubsection{Margin Bound}
Recall the per-round regret is 
\begin{align*}
    \Delta_{a_t}(t)  &= x_{a_t^*}^\top(t) \beta^* - x_{a_t}^\top (t) \beta^*\\
    & = x_{a_t^*}^\top (t) \beta^*  - x_{a_t^*}^\top(t) \hat{\beta}_{\ell(t)} + \underbrace{(x_{a_t^*}^\top(t) \hat{\beta}_{\ell(t)} -  x_{a_t}^\top (t) \hat{\beta}_{\ell(t)})}_{\leq 0} +  x_{a_t^*}^\top (t) \hat{\beta}_{\ell(t)} - x_{a_t}^\top (t) \beta^*\\
    & \leq \norm{x_{a_t^*}(t)}_\infty \Vert \hat{\beta}_{\ell(t)} - \beta^*\Vert_1 + \norm{x_{a_t} (t)}_\infty \Vert \hat{\beta}_{\ell(t)} - \beta^*\Vert_1\\
    & \leq 2 \phi_{\ell(t)},
\end{align*}
where 
\[
\phi_{\ell(t)} \asymp \xmax \sigma^2 \kappa_0 \left( \sqrt{\frac{(s + s^*)^2 \log d}{N_{\ell-1}}} + \sqrt{\frac{(s + s^*)^3 (\log d)^2 \log(1/\delta) \log^4 N_{\ell-1}}{N_{\ell-1}^2 \epsilon^2}}\right).
\]

Define the event 
\[
\cM_t : = \left\{ x_{a_t^*}^\top \beta^* > \max_{i \neq a_{t}^*} x_{a_t}^\top \beta^* + h_{t} \right\}.
\]
Under $\cM_t \cap \cA_{\ell(t)}$, we have the following for any $i \neq a_t^*$:
\begin{align*}
    x_{a_t^*}^\top(t) \hat{\beta}_{\ell(t)} - x_i^{\top}(t) \hat{\beta}_{\ell(t)} & = \innerprod{x_{a_t^*}(t), \hat{\beta}_{\ell(t)} - \beta^*} + \innerprod{x_{a_t^*}(t) - x_i(t), \beta^*} + \innerprod{x_i(t), \beta^* - \hat{\beta}_{\ell(t)}}\\
    & \geq - \phi_{\ell(t)} + h_{t-1} -  \phi_{\ell(t)}.
\end{align*}

Thus, if we set $h_{t} = 3  \phi_{\ell(t)}$ then $ x_{a_t^*}^\top(t) \hat{\beta}_{\ell(t)} - \max_{i \neq a_t^*}x_i^{\top}(t) \hat{\beta}_{\ell(t)} \geq \phi_{\ell(t)}$. As a result, in $t$th round the regret is 0 almost surely as the optimal arm will be chosen with probability 1. Therefore,
\vspace{1cm}
 \begin{align*}
     \bbE(\Delta_{a_t}(t)) & = \bbE\{ \Delta_{a_t}(t) \ind_{\cM_t^c}\}\\
     & = \bbE\{ \Delta_{a_t}(t) \ind_{\cM_t^c \cap \cE_{\ell(t)} \cap \cG_{\ell(t)}}\} + \bbE\{ \Delta_{a_t}(t) \ind_{\cM_t^c \cap (\cE_{\ell(t)} \cap \cG_{\ell(t)})^c}\}\\
     & \leq 2 \phi_{\ell(t)} \pr (\cM_t^c) + 2 \xmax \bmax \pr(\cE_\ell(t)^c \cup \cG_\ell(t)^c) \\
     & \leq 2  \phi_{\ell(t)} \pr(\cM_t^c) + 2 \xmax \bmax \pr(\cA_{\ell(t)}^c)\\
     & \le 2 \phi_{\ell(t)} \left(\frac{3 \phi_{\ell(t)}}{\Delta_*}\right)^\alpha +  2 \xmax \bmax \pr(\cA_{\ell(t)}^c)
    \end{align*}

\paragraph{Case 1- $\alpha = 0$}:
Discussed in the problem interdependent case.

\paragraph{Case 2 - $\alpha \in (0,1)$:}
\begin{align*}
    &\bbE \{R(T)\}  \le 2 \xmax\bmax (2^L -1) + \\
    & 24 \xmax \sigma^2 \kappa_0 \underbrace{\sum_{\ell \ge L} \Delta_*^{-\alpha} N_{\ell-1} \left[ \left(\frac{(s + s^*)^2 \log d}{N_{\ell-1}}\right)^{(1+\alpha)/2} + \left(\frac{(s + s^*)^3 (\log d)^2 \log(1/\delta) \log^4 N_{\ell-1}}{N_{\ell-1}^2 \epsilon^2}\right)^{(1+\alpha)/2}\right]}_{I_\alpha} + O(1). 
\end{align*}

Straightforward algebra leads to 
\[
I_\alpha \lesssim \frac{(s + s^*)^{1+\alpha} (\log d)^{(1+\alpha)/2}}{\Delta_*^\alpha} \left(\frac{T^{\frac{1-\alpha}{2}} - 1}{1-\alpha}\right) + \Psi_\alpha \Delta_*^{-\alpha}\epsilon^{-(1+\alpha)} \{\log(1/\delta)\}^{(1+\alpha)/2}\left((s + s^*)^3 \log^2 d \right)^{\frac{1+\alpha}{2}} ,
\]
where $\Psi_\alpha = \frac{1 - T^{-\alpha}}{1 - 2^{-\alpha}} (\log T)^{2\alpha + 2}$.

\paragraph{Case 3 - $\alpha =1 $:}
In this case, taking $\alpha \to 1$ in the previous bound yields
\[
I_1 \lesssim \frac{(s + s^*)^2 \log d}{\Delta_*} \log T + \frac{\Psi_1}{\Delta_* \epsilon^2} \log(1/\delta) (s + s^*)^3 \log^2 d.
\]

\paragraph{Case 4- $\alpha >1$}
In this case, we get
\[
I_\alpha \lesssim \frac{(s + s^*)^{1+\alpha} (\log d)^{(1+\alpha)/2}}{\Delta_*^\alpha} \left(\frac{1 - T^{-\frac{\alpha-1}{2}} }{\alpha  - 1}\right) + \Psi_\alpha \Delta_*^{-\alpha}\epsilon^{-(1+\alpha)} \{\log(1/\delta)\}^{(1+\alpha)/2}\left((s + s^*)^3 \log^2 d \right)^{\frac{1+\alpha}{2}} ,
\]

\paragraph{Case 5- $\alpha = \infty$}
In this case, we will ensure that $3 \phi_{\ell(t)}/\Delta_* <1$ for all $ t \geq t_L$. Therefore,  it suffices to have 
\[
\ell(t) \gtrsim (s + s^*)^2 \log d + \frac{(s + s^*)^3 (\log d)^2 \log(1/\delta)}{\epsilon^2}.
\]

Under the above condition, we have $I_\infty = 0$ and the result follows.

\section{\uppercase{Numerical Experiments: More details}}
\label{sec: numerical details appendix}
We discuss the numerical experiments in further details and also provide additional experiments for (i) different noise distribution, (ii) effect of the tuning parameter $s$ and (iii) effect of the number of arms $K$. For implementation of FLIPHAT, we used three hyperparameters $s$ (sparsity guess), $\eta$ (step size for N-IHT) and $M>0$ (to control the number of steps in N-IHT with $M_\ell = M \log N_\ell$ at the $\ell$-th episode). For these synthetic experiments, we used $\xmax$ and $\bmax$ based on the true underlying setting and chose $C, R_\ell$ accordingly (see Algorithm \ref{alg: fliphat} and Proposition \ref{prop: estimation episode}. In practice, the contexts can be rescaled (to eliminate $\xmax$) and $\bmax$ could either be chosen based on existing domain knowledge or held-out data or included as another hyperparameter. The task of tuning such hyperparameters can be tricky because of privacy constraints, but we do not focus on that aspect in this paper and keep it for future work.

\begin{figure}
    \centering
    \includegraphics[width=0.48\linewidth]{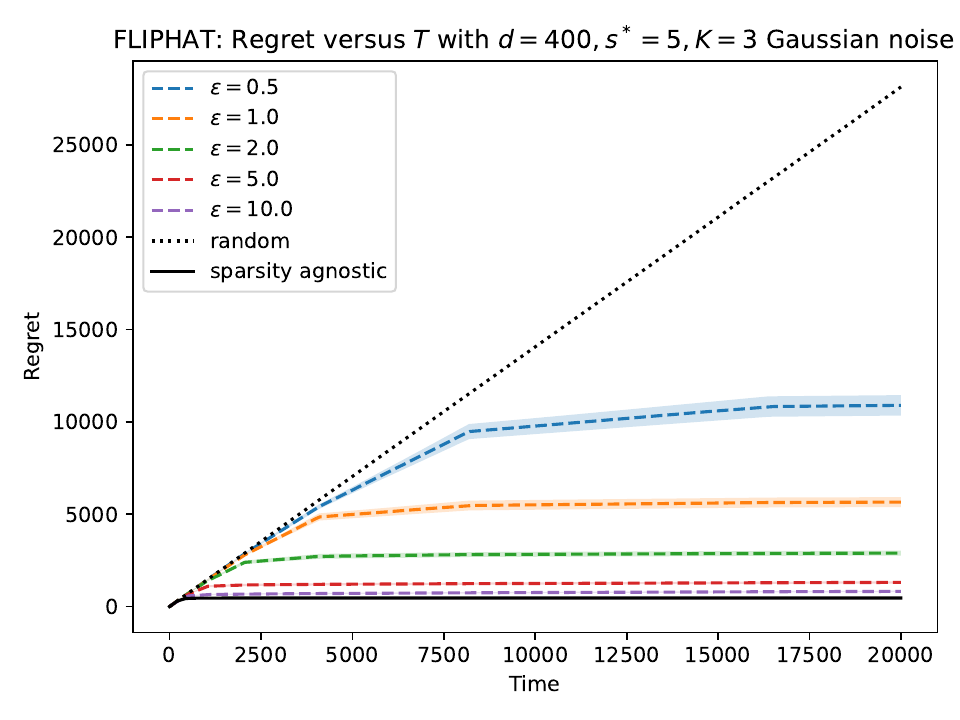}
    \includegraphics[width=0.48\linewidth]{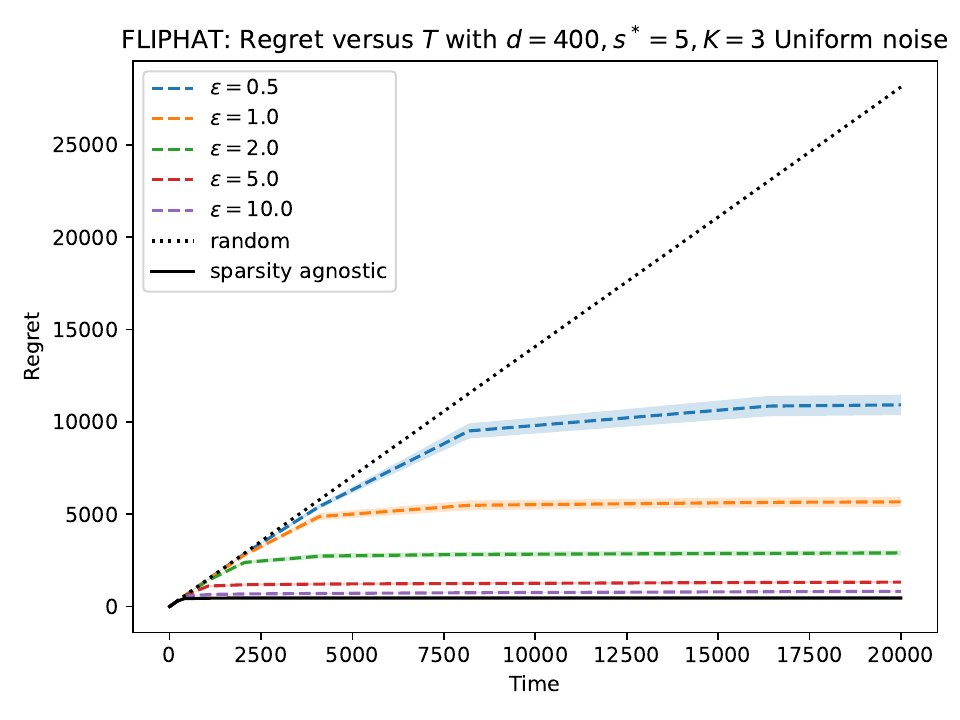}
    \caption{Regret $R(T)$ versus time horizon $T$ in the setting $d=400, s^*=5, K=3$, with different privacy parameters $\epsilon$ and $\delta=0.01$. Contexts iid from normal with autoregressive covariance. (Left) Gaussian observation noise, mean 0 and $\sigma=0.1$, (Right) Bounded observation noise, drawn from $\text{Uniform}(-0.1,0.1)$}
    \label{fig:regret-vs-T}
\end{figure}
\subsection{Regret vs $T$}
We perform two experiments to study how the regret $R(T)$ depends on time horizon $T$ for different levels of the privacy parameter $\epsilon$. For both experiments, we consider $d=400, s^*=5, K=3$ and choose the privacy parameters $\epsilon\in\{0.5,1,2,5,10\}$ and $\delta=0.01.$ The contexts were drawn i.i.d. from $\cN(0,\Sigma)$ with an auto-regressive covariance matrix $\Sigma_{ij}=0.1^{|i-j|}$. The rewards were sampled from the linear model \eqref{eq: base model} -- in the first experiment $\epsilon(t)\sim \cN(0,\sigma)$ with $\sigma=0.1$ and the second experiment, $\epsilon(t)\sim \text{Uniform}(-0.1,0.1)$. The hyperparameters were chosen as $s=10, \eta=10^{-4}, M=1.6$. Each experiment was repeated 60 times and the mean and 95\% confidence intervals are shown in Figure \ref{fig:regret-vs-T}. We find the behavior very similar in both cases. For additional perspective, we added the random bandit (full privacy $\epsilon=0$) and sparsity agnostic bandit (no privacy $\epsilon\to\infty$). We note that even for $\epsilon=0.5$, we observe clear sub-linear behavior of the regret, particularly after $T=7500$. For larger $\epsilon$, we observe that the performance is similar to the sparsity agnostic bandit. To our knowledge, this current paper is the first work on joint differentially private algorithm for sparse contextual bandits, as a result, we could not find any other suitable baseline to compare the performance with.

\subsection{Regret vs $d$}
The setup is exactly similar to the first two experiments, discussed above. In this case, we only collect the regret suffered at $T=10000$, but use different context dimensions $d$, to examine the effect of $d$ on the regret for different choices of $\epsilon$. Note that $s^*$ is fixed in these experiments. The result is shown in Figure \ref{fig: regret} (right) in the main text. We wish to remark on the distinct behavior of FLIPHAT, compared to the JDP contextual bandit in \citep{shariff2018differentially}. In this work, the authors explore joint differential privacy for contextual bandits, with no additional condition on the sparsity -- as a result, the regret scales as $d^2$, which is illustrated in Figure 3 (in the Supplementary Material of \citep{shariff2018differentially}). In our case, we demonstrate that with the additional assumption of sparsity, we can improve this dependence on $d$ under joint differential privacy -- this is illustrated in Figure \ref{fig: regret} (right), where we observe that $R_d(T) \approx \alpha + \beta \log d$, where $\beta$ (for different $\epsilon$) is estimated using a linear model and shown in the legend.

\subsection{Effect of $s$}\label{app: numerical s}
In this experiment, the goal was to examine the effect of one of the more important hyperparameters $s$. According to Theorem \ref{thm: regret}, as long as $s\gtrsim \kappa^2 s^*$, the upper bound (in the problem independent setting of $\alpha=0$), scales as $(s+s^*)^{3/2}$. As discussed in the remarks following the theorem, setting $s$ in the order of $o(d)$ guarantees improved dependence on $d$, with $s\asymp s^*$ (same order) yielding near-optimal dependence on $s^*$. In practice, $s^*$ is not known, although some domain knowledge might give reasonable lower (and upper) bounds for it. Figure \ref{fig:regret-s-K} (left) shows $\log R(T)$ against $\log s$ for $T=10000$ in the same setting as the first experiment (with the only change being $s^*=10$ here) for different choices of $s$ ranging from 5 to 60 (and different $\epsilon$). We observe that the regret is poor for lower values and is best around $s=s^*$ and then worsens a bit. The performance is reasonably robust for higher values of $s$ under high privacy (small $\epsilon$) regime. Under low privacy regime, the log-regret scales only linearly in $\log s$. The slope of the $\log R(T)$ vs $\log s$ lines for each $\epsilon$ in the part $s\geq s^*$ is estimated by fitting a linear model and displayed in the corresponding legend. 

\subsection{Effect of $K$}
In our analysis, we held the number of arms $K$ fixed (small compared to $d, T$). We conduct numerical experiment to see the effect of $K$ on $R(T)$ for FLIPHAT with $T=10000$, in the setting of the first experiment $d=400, s^*=5$, at different levels of privacy. The results are shown in Figure \ref{fig:regret-s-K} (right), where we notice that the dependence is indeed sub-linear.

\begin{figure}
    \centering
    \includegraphics[width=0.48\linewidth]{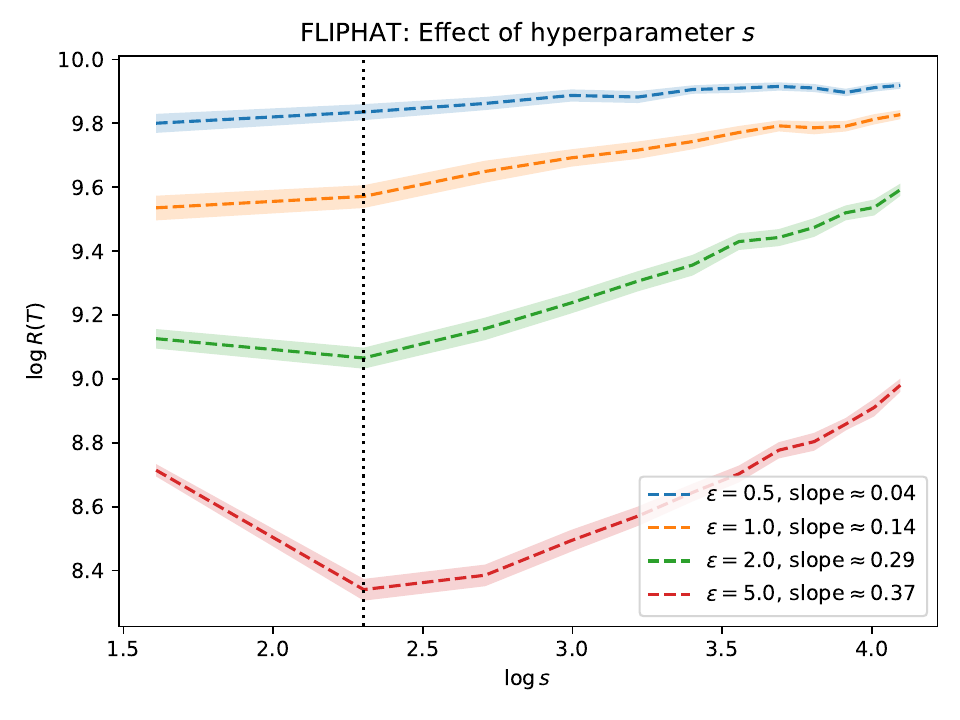}
    \includegraphics[width=0.48\linewidth]{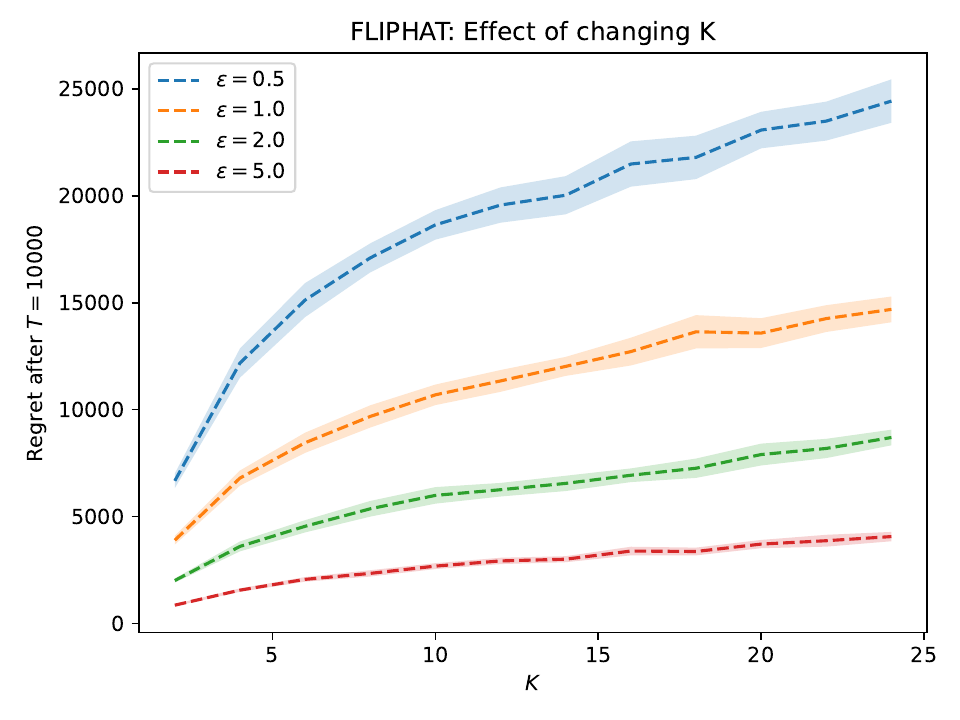}
    \caption{(Left) Effect of the choice of the tuning parameter $s$ on regret $R(T)$, with $T=10000$ - true $s^*=10$ - for different privacy levels. (Right) Regret at $T=10000$ for different number of arms $K$, keeping $d, s^*$ and other parameters fixed.}
    \label{fig:regret-s-K}
\end{figure}